\documentclass[nohyperref]{article}

\usepackage{microtype}
\usepackage{graphicx}
\usepackage{subfigure}
\usepackage{booktabs} 
\usepackage{multicol}
\usepackage{multirow}
\usepackage[makeroom]{cancel}
\usepackage[english]{babel}
\usepackage{color}
\usepackage{hyperref}

\usepackage[accepted]{icml2022}

\usepackage{amsmath}
\usepackage{amssymb}
\usepackage{mathtools}
\usepackage{amsthm}

\usepackage[capitalize,noabbrev]{cleveref}

\theoremstyle{plain}
\newtheorem{theorem}{Theorem}[section]
\newtheorem{proposition}[theorem]{Proposition}

\newtheorem{corollary}[theorem]{Corollary}
\theoremstyle{definition}

\newtheorem{assumption}[theorem]{Assumption}
\theoremstyle{remark}
\newtheorem{remark}[theorem]{Remark}

\usepackage[textsize=tiny]{todonotes}

\newcommand{\clip}{\mathrm{clip}}
\newcommand{\KL}{D_{\mathrm{KL}}}
\newcommand{\TV}{D_{\mathrm{TV}}}
\newcommand{\TVmax}{D_{\mathrm{TV}}^{\mathrm{max}}}
\newcommand{\KLmax}{D_{\mathrm{KL}}^{\mathrm{max}}}
\newcommand{\norm}[1]{\left\lVert#1\right\rVert}

\newcommand{\red}[1]{\textcolor{red}{#1}}
\newcommand{\blue}[1]{\textcolor{blue}{#1}}

\definecolor{orange}{rgb}{1,0.5,0}

\icmltitlerunning{You May Not Need Ratio Clipping in PPO}

\begin{document}

\twocolumn[
\icmltitle{You May Not Need Ratio Clipping in PPO}



\icmlsetsymbol{equal}{*}

\begin{icmlauthorlist}
\icmlauthor{Mingfei Sun}{ox,ms}
\icmlauthor{Vitaly Kurin}{ox}
\icmlauthor{Guoqing Liu}{ms}
\icmlauthor{Sam Devlin}{ms}
\icmlauthor{Tao Qin}{ms}
\icmlauthor{Katja Hofmann}{ms}
\icmlauthor{Shimon Whiteson}{ox}
\end{icmlauthorlist}

\icmlaffiliation{ox}{Department of Computer Science, University of Oxford}
\icmlaffiliation{ms}{Microsoft Research}

\icmlcorrespondingauthor{Mingfei Sun}{mingfei.sun@cs.ox.ac.uk}

\icmlkeywords{Machine Learning, ICML}

\vskip 0.3in
]



\printAffiliationsAndNotice{}  

\begin{abstract}
Proximal Policy Optimization (PPO) methods learn a policy 
by iteratively performing multiple mini-batch optimization epochs of a surrogate objective 
with one set of sampled data. 
Ratio clipping PPO is a popular variant that clips the probability ratios between the target policy 
and the policy used to collect samples. 
Ratio clipping yields a pessimistic estimate of the original surrogate objective, 
and has been shown to be crucial for strong performance. 
We show in this paper that such ratio clipping may not be a good option 
as it can fail to effectively bound the ratios.
Instead, one can directly optimize the original surrogate objective for	 multiple epochs; the key is to find a proper condition to early stop the optimization epoch in each iteration.
Our theoretical analysis sheds light on how to determine when to stop the optimization epoch,
and call the resulting algorithm Early Stopping Policy Optimization (ESPO). 
We compare ESPO with PPO across many continuous control tasks 
and show that ESPO significantly outperforms PPO. 
Furthermore, we show that ESPO can be easily scaled up to distributed training with many workers, 
delivering strong performance as well. 
\end{abstract}

\section{Introduction}
Proximal Policy Optimization~\citep{schulman2017proximal} methods learn a policy by alternating between two procedures: 
sampling data from the policy, 
and performing multiple epochs of policy optimization on the sampled data -- 
the \emph{epoch} here refers to one round of mini-batch policy update on the full batch data. 
These multiple epochs differentiate PPO from methods based on the policy gradient theorem~\citep{sutton2000policy} 
and could improve sample efficiency in practice. 
PPO has two main variants: 
Kullback-Leibler (KL) regularized policy optimization (KL-PPO), 
and ratio clipping policy optimization (RC-PPO). 
The first variant is a direct result of the monotonic improvement guarantee, as used in Trust Region Policy Optimization (TRPO)~\citep{schulman2015trust},
which describes a way to monotonically improve the policy via a KL-regularized surrogate objective. 
The use of such regularization is crucial~\citep{neu2017unified}, 
as it clearly defines a trust region constraint in each optimization epoch~\citep{schulman2017proximal}, 
and has been shown to be essential for the converngece of PPO~\citep{liu2019neural}.

By contrast, RC-PPO has no explicit regularization in its optimization:
it merely clips the ratios between the policy to be optimized and the policy used to collect samples in each epoch. 
The objective with these clipped probability ratios forms a pessimistic estimate of the surrogate objective proposed in Conservative Policy Iteration (CPI)~\citep{kakade2002approximately}. 
Compared to KL-PPO and TRPO, ratio clipping avoids expensive computation of the KL divergence
and thus scales to large problem settings~\citep{berner2019dota,ye2020mastering}. 
It also yields stronger empirical performance than its KL regularized counterpart~\citep{schulman2017proximal}.

However, it remains largely unclear how ratio clipping relates to KL regularization, 
and whether/how a trust region constraint is implicitly enforced via ratio clipping. 
Existing studies, including the original PPO paper~\citep{schulman2017proximal}, 
interpret this ratio clipping as a way to bound the ratios~\citep{queeney2021generalized,hessel2021muesli}. 
Many other studies, however, suggest contradictory conclusions: 
the ratios could become unbounded via ratio clipping~\citep{wang2020truly,engstrom2020implementation,tomar2020mirror}, 
and, even if the ratios are bounded, a KL-based trust region constraint could still be violated~\citep{wang2020truly}. 
Furthermore, several studies point out that RC-PPO may benefit from code-level optimizations,
rather than ratio clipping itself~\citep{engstrom2020implementation,andrychowicz2020matters}.

In this paper, we show that ratio clipping may not be needed. 
One can instead directly optimize the surrogate objective adopted in CPI. 
The key is to find a proper condition to early stop the multiple epochs of surrogative objective optimization. 
Specifically, we first show that ratio clipping itself is not sufficient to 
remove the incentive of updating ratios beyond the clipping range. 
In fact, the ratios could grow without bound as the optimization epoch proceeds.
We then show that, even though the ratios are bounded with 
a proper number of optimization epochs and a learning rate annealing strategy, 
the resulting total variation (TV) divergence is not effectively bounded. 

We also propose a novel policy optimization method, 
Early Stopping Policy Optimization (ESPO), 
which simply computes a ratio deviation to effectively determine 
when to stop the optimization epoch. 
The proposed ratio deviation is theoretically derived from a novel monotonic improvement guarantee and
can be effectively estimated without resorting to any complex computations, e.g., KL in TRPO. 
We evaluate the performance of ESPO in many high-dimensional control benchmarks 
and show that ESPO consistently outperforms PPO. 
Moreover, we show that ESPO scales up easily in distributed training 
to many workers~\citep{espeholt2018impala} and its performance is similar to or better than that of Distributed PPO~\citep{heess2017emergence}.

\section{Related work}
\paragraph{Regularized Policy optimization.}
Regularized policy optimization  learns a policy directly with a regularized objective function.
Kakade and Langford~\yrcite{kakade2002approximately} show that,
with the objective defined as \emph{policy advantage}, a.k.a.\ a surrogate objective,
one could improve a policy by linearly mixing it 
with the policy that optimizes the surrogate objective,
i.e., Conservative Policy Iteration (CPI)~\citep{kakade2002approximately}. 
\citet{pirotta2013safe} propose the optimal linear mixing that 
guarantees a robust monotonic improvement. 

Trust Region Policy Optimization (TRPO)~\citep{schulman2015trust} relaxes this restrictive linear mixing 
by using the Total Variation (TV) divergence between two stochastic policies for regularization. 
The same regularization was also adopted in Constrained Policy Optimization (CPO)~\cite{achiam2017constrained}. 
With the TV divergence as regularization, 
TRPO and CPO present two types of monotonic improvement guarantee, 
one relying on the maximum TV over all state-action pairs~\citep{schulman2015trust} and 
the other relying on the average TV divergence over all empirical samples~\citep{achiam2017constrained}. 
As TV divergence can be upper bounded by KL divergence, 
in practice, TRPO relies on the KL divergence to update the policies.  
This KL divergence has thus been used for theoretical analysis in many studies,
e.g.,~\citet{liu2019neural,wang2020truly}. 

Recently, a more generic regularization has been proposed in Mirror Descent Policy Optimization (MDPO)~\citep{tomar2020mirror}, 
which uses a Bregman divergence to regularize the policy update, providing a unified view of the regularization used in various policy optimization methods. 
See ~\citet{geist2019theory} for more regularization theory.

\paragraph{Clipping and trust region constraint.}
Unlike the explicit regularization used in regularized policy optimization methods, 
Proximal Policy Optimization (PPO)~\citep{schulman2017proximal} exploits a simple scheme: 
clipping probability ratios between policies at each optimization epoch. 
Although this scheme produces strong empirical performance, 
it remains largely unknown how ratio clipping relates to trust region enforcement.
Existing studies, including the original PPO paper~\citep{schulman2017proximal}, 
related this clipping scheme to the bounding of ratios~\citep{queeney2021generalized,hessel2021muesli}. 
Many other studies, however, suggest the ratios could become unbounded via this clipping~\citep{wang2020truly,engstrom2020implementation,tomar2020mirror}. 
\citet{wang2020truly} show that, even if the ratios are bounded, the trust region constraint, defined with the KL as in TRPO, 
could still be violated. 
Ratio clipping has also been interpreted 
as a way to equivalently enforce a TV divergence~\cite{queeney2021generalized,hessel2021muesli}
or approximately apply a hinge loss~\cite{yao2021hinge}. 

In this paper, we show that a TV divergence can only be enforced by this clipping scheme 
if all hyperparameters are properly set. 
\citet{akrour2019projections} point out that the use of clipping in PPO 
can remove the incentive to change the policy past a certain total variation threshold,
and thus implicitly enforces a trust region constraint 
if the step-size of gradient descent is properly set. 
However, we show that, even with proper hyperparameter sweeping, 
clipping can still fail as the number of optimization epochs increases in an iteration.

\citet{hessel2021muesli} apply clipping to the advantage function 
when updating the policy via Maximum a Posterior policy Optimization (MPO). 
The use of a clipped advantage rigorously defines a total variation distance between 
the posterior policy and the prior. 
However, clipping the advantage is different from clipping ratios, 
and their analysis does not apply to ratio clipping PPO.

\section{Background}
We model reinforcement learning as an infinite-horizon discounted Markov decision process (MDP) 
$\{\mathcal{S}, \mathcal{A}, P, r, p_0\}$, 
where $\mathcal{S}$ is a finite set of states,
$\mathcal{A}$ is a finite set of actions,
$P: \mathcal{S} \times \mathcal{A} \times \mathcal{S} \rightarrow \mathbb{R}$ 
is the transition probability distribution, 
$r: \mathcal{S}\times\mathcal{A}\rightarrow\mathbb{R}$ is the reward function, 
and $p_0:\mathcal{S}\rightarrow\mathbb{R}$ is the initial state distribution. 
The performance for a policy $\pi(a| s)$ is defined as:
$J(\pi) \triangleq \mathbb{E}_{s_0\sim p_0}\big[ \sum_{t=0}^{\infty}\gamma^t r(s_t, a_t ) \big]$, 
where $\gamma\in [0, 1)$ is the discount factor. 
The action-value function $Q_\pi$ 
and value function $V_\pi$ are defined as:
\begin{align*}
Q_{\pi}(s_t, a_t) &\triangleq \mathbb{E}_{\substack{s_{t+1}\sim p(\cdot|s_t,a_t),\\a_{t+1}\sim\pi(\cdot|s_{t+1})}}\Big[ \sum_{l=0}^{\infty} \gamma^l r(s_{t+l}, a_{t+l}) \Big], \\
V_{\pi}(s_t) &\triangleq \mathbb{E}_{a_t\sim\pi(\cdot|s_t)}\Big[Q_{\pi}(s_t, a_t) \Big]. 
\end{align*}
Define the advantage function as $A_{\pi}(s, a) \triangleq Q_{\pi}(s, a) - V_{\pi}(s)$. 
Let the discounted state distribution be 
$d_{\pi}(s) \triangleq (1-\gamma) \sum_{t=0}^{\infty} \gamma^t P(s_t=s|\pi)$, 
which allows us to rewrite the performance of policy $\pi$ as 
\begin{equation*}
    J(\pi) = \frac{1}{1-\gamma} \mathbb{E}_{s\sim d_{\pi}, a\sim\pi}[r(s, a)]. 
\end{equation*}
We use $(s, a)\sim d_{\pi}$ to denote $s\sim d_{\pi}, a\sim\pi$. 
Moreover, the following useful identity expresses the expected return of another policy $\tilde{\pi}$ 
in terms of the advantage over $\pi$~\citep{kakade2002approximately}:
\begin{equation}\label{equ:performance-identity}
J(\tilde{\pi}) = J(\pi) + \frac{1}{1-\gamma} \sum_{s}d_{\tilde{\pi}}(s)\sum_{a}\tilde{\pi}(a|s) A_{\pi}(s, a), 
\end{equation}
where $d_{\tilde{\pi}}(s)$ is the discounted state distribution induced by $\tilde{\pi}$.
The complex dependency of $d_{\tilde{\pi}}(s)$ on $\tilde{\pi}$ 
makes the righthand side difficult to optimize directly. 
\citet{schulman2015trust} proposed to consider the following surrogate objective
\begin{align*}
L_{\pi}(\tilde{\pi}) &= J(\pi) + \frac{1}{1-\gamma} \sum_{s}d_{\pi}(s)\sum_{a}\tilde{\pi}(a|s)A_{\pi}(s, a) \\
&= J(\pi) + \frac{1}{1-\gamma} \mathbb{E}_{(s, a)\sim d_\pi}\Big[\frac{\tilde{\pi}(a|s)}{\pi(a|s)}A_{\pi}(s, a) \Big],
\end{align*}
where the $d_{\tilde{\pi}}$ is replaced with $d_{\pi}$. 
Define $\TVmax(\pi, \tilde{\pi})\triangleq \max_{s} \TV\big(\pi(\cdot|s), \tilde{\pi}(\cdot|s)\big)$, 
where $\TV$ is the total variation (TV) divergence. 
\begin{theorem}\label{theo:trpo-theorem}
(Theorem 1 in \citet{schulman2015trust}) Let $\alpha = \TVmax(\pi, \tilde{\pi})$. 
Then the following bound holds
\begin{equation*}
J(\tilde{\pi}) \geq L_{\pi}(\tilde{\pi}) - \frac{4\xi \gamma}{(1-\gamma)^2} \alpha^2,
\end{equation*}
where $\xi=\max_{s,a}\lvert A_\pi(s, a)\rvert$. 
\end{theorem}
This theorem forms the foundation of policy optimization methods, 
including Trust Region Policy Optimization (TRPO)~\citep{schulman2015trust} 
and Proximal Policy Optimization (PPO)~\citep{schulman2017proximal}. 
TRPO suggests a robust way to take large update steps by using a constraint, 
rather than a regularization, over the policies. 
Specifically, as the TV divergence can be upper bounded by the KL divergence, 
i.e., $[\TVmax(\pi, \tilde{\pi})]^2 \leq \frac{1}{2} \KLmax(\pi, \tilde{\pi}) $, 
TRPO constrains the KL divergence instead. 
It also leverages the average KL divergence rather than the max divergence 
as the former can be more easily estimated from samples:
\begin{align}\label{equ:trpo-objective}
\max_{\tilde{\pi}}& \quad \mathbb{E}_{(s, a)\sim d_{\pi}}\Big[\frac{\tilde{\pi}(a|s)}{\pi(a|s)}A_{\pi}(s, a) \Big], \nonumber \\
\text{s.t.}& \quad \mathbb{E}_{s\sim d_{\pi}} \big[ \KL(\pi(\cdot| s), \tilde{\pi}(\cdot| s)) \big] \leq \beta, 
\end{align}
where $\beta$ is a hyperparameter. 
Still, this constrained optimization is complicated as it requires 
using conjugate gradient algorithms to approximate KL divergence constraint. 
KL regularized PPO (KL-PPO) instead considers the following regularized policy optimization with average KL divergence
\begin{equation}\label{equ:kl-ppo-objective}
\max_{\tilde{\pi}} \mathbb{E}_{(s, a)\sim d_{\pi}}\Big[\frac{\tilde{\pi}(a|s)}{\pi(a|s)}A_{\pi}(s, a) - \beta \KL(\pi(\cdot| s), \tilde{\pi}(\cdot| s))\Big]. 
\end{equation}
Ratio clipping PPO (RC-PPO) simplifies the above by clipping ratios 
$\lambda_{\tilde{\pi}}\triangleq\frac{\tilde{\pi}(a|s)}{\pi(a|s)}$
to form a lower bound of $L_{\pi}(\tilde{\pi})$: 
\begin{equation}\label{equ:ppo-objective}
\max_{\tilde{\pi}} \mathbb{E}_{d_{\pi}} \big[ \min\big(\lambda_{\tilde{\pi}} A_{\pi}, \clip(\lambda_{\tilde{\pi}}, 1-\epsilon, 1+\epsilon)A_{\pi} \big) \big], 
\end{equation}
where $\epsilon$ is a hyperparameter for the clipping range.

\section{Understanding ratio clipping}
PPO performs multiple mini-batch updates (i.e., epochs) of policy updates on the same set of samples.
With proper regularization in each optimization epoch, e.g., a trust region constraint, monotonic improvement can be guaranteed~\citep{neu2017unified,geist2019theory}. 
We thus first investigate whether ratio clipping can enforce a trust region constraint in those optimization epochs. 
We also point out discrepancies between the existing theory and practice.

\subsection{Condition to enforce a trust region constraint}
The first question about ratio clipping is 
whether any trust region constraint is effectively enforced in each training epoch. 
To investigate this, we first check the monotonic improvement guarantee in TRPO~\citep{schulman2015trust}. 
Specifically, the improvement guarantee underlying TRPO imposes a KL divergence regularization 
between the policy to be optimized and the policy used to collect samples. 
For RC-PPO, \citet{wang2020truly} show that the KL divergence 
can never be bounded between policies in optimization epochs.
However, the monotonic improvement guarantee adopted in TRPO effectively relies on the TV divergence, 
(as defined in Theorem~\ref{theo:trpo-theorem}), 
not on the KL divergence. 
To be more specific, Theorem~\ref{theo:trpo-theorem} requires the maximum TV divergence over the whole state space to be bounded. 
We have the following:
\begin{proposition}\label{prop:ratio-tv-bound}
If $\lambda(s, a) \triangleq \frac{\tilde{\pi}(a|s)}{\pi(a|s)}$ is bounded by $M_1$ and $M_2$, 
i.e., $M_1 \leq \lambda(s, a) \leq M_2$ for $\forall (s, a)$ 
where $M_1\in(0, 1)$ and $M_2\in(1,\infty)$, 
then the following bound holds:
$\TVmax\big(\pi(\cdot|s), \tilde{\pi}(\cdot|s)\big) \leq \min(\frac{1}{M_1} -1 , M_2 -1)$. 
\end{proposition}
We provide the proof in Appendix~\ref{app:ratio-tv-bound}. 
As $\TV$ is a bounded divergence between $[0, 1]$, 
the ratio guarantee makes sense when $M_1\geq 0.5$ or $M_2 \leq 2.0$. 
This corollary provides a sufficient condition to 
constrain the TV divergence via bounding the ratios. 
Unlike the KL divergence used in TRPO, 
the maximum TV divergence can be well bounded as long as the ratios are bounded. 

\begin{remark}
If ratios are well bounded between $0.5$ and $2.0$, a trust region defined by TV divergence is then effectively enforced. 
\end{remark}

\subsection{Does ratio clipping satisfy the condition?}
The PPO paper~\citep{schulman2017proximal} suggests the following clipping scheme: 
given a hyperparameter for clipping $\epsilon$, 
the clipping range is then given by $[1-\epsilon, 1+\epsilon]$. 
\citet{schulman2017proximal} propose this ratio clipping to remove the incentive 
for updating the ratio out of some range. 
If the ratios are to be bounded between $[1-\epsilon, 1+\epsilon]$ ($\epsilon<1$), 
then according to Proposition~\ref{prop:ratio-tv-bound}, we have
\begin{equation*}
\TVmax(\pi, \tilde{\pi}) \leq \min(\frac{\epsilon}{1-\epsilon}, \epsilon) \leq \epsilon. 
\end{equation*}
Proposition~\ref{prop:ratio-tv-bound} suggests another way to bound ratios:
$\frac{1}{1+\epsilon} \leq \lambda \leq 1+\epsilon$. 
The resulting objective is:
\begin{equation}\label{equ:new-clipping-objective}
\max_{\tilde{\pi}} \mathbb{E}_{d_{\pi}} \big[ \min\big(\lambda_{\tilde{\pi}} A_{\pi}, \clip(\lambda_{\tilde{\pi}}, \frac{1}{1+\epsilon}, 1+\epsilon)A_{\pi} \big) \big], 
\end{equation}
where $\epsilon$ is a hyperparameter to specify the clipping range. 

\begin{figure}
    \centering
    \includegraphics[width=0.75\linewidth]{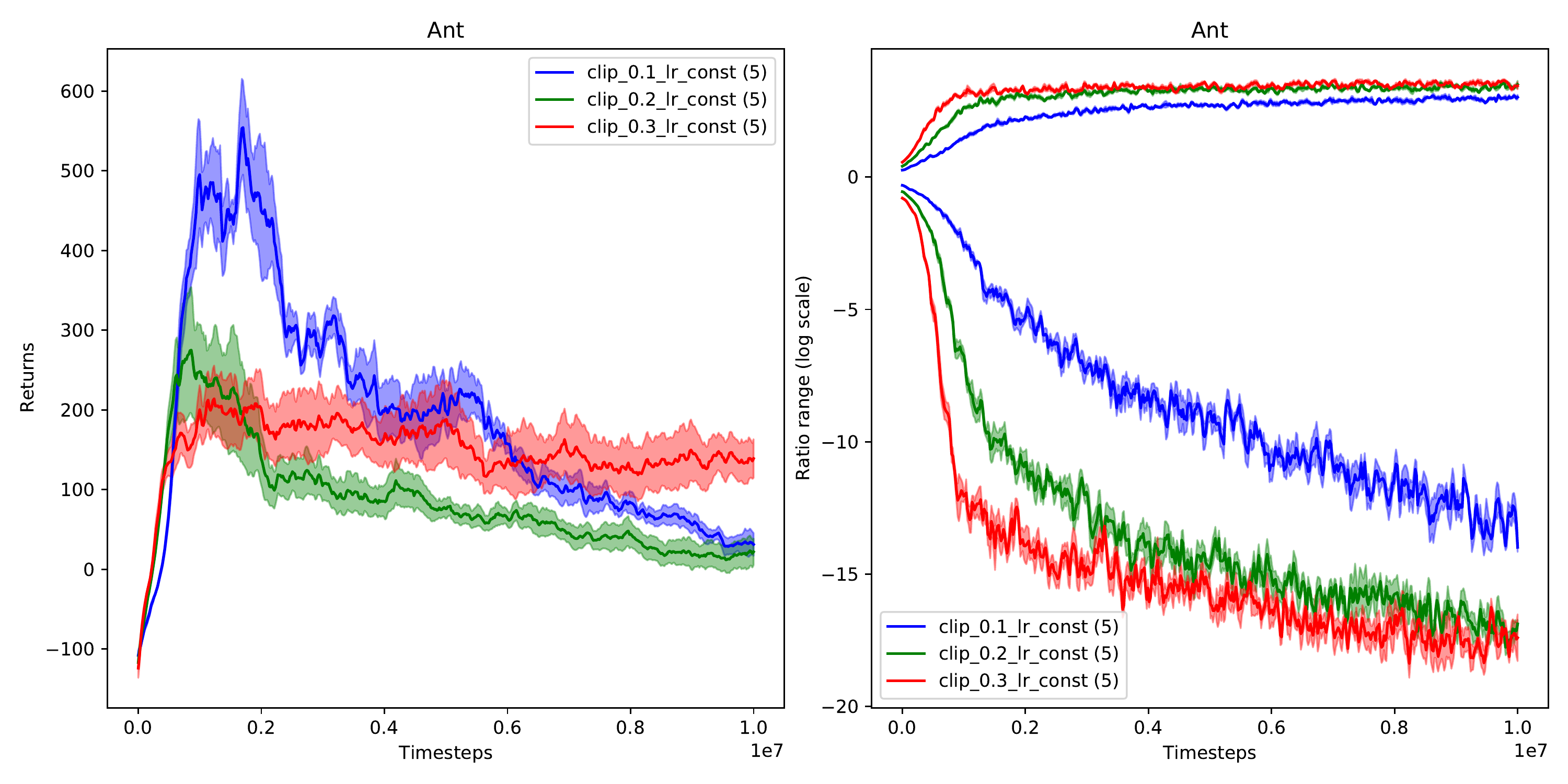}
    \caption{Empirical returns and ratio ranges (log scale) of RC-PPO trained on Ant-v2 (with constant learning rate $0.0003$). }
    \label{fig:ant-ratio-range-no-decay}
\end{figure}

\begin{figure}
    \centering
    \includegraphics[width=0.75\linewidth]{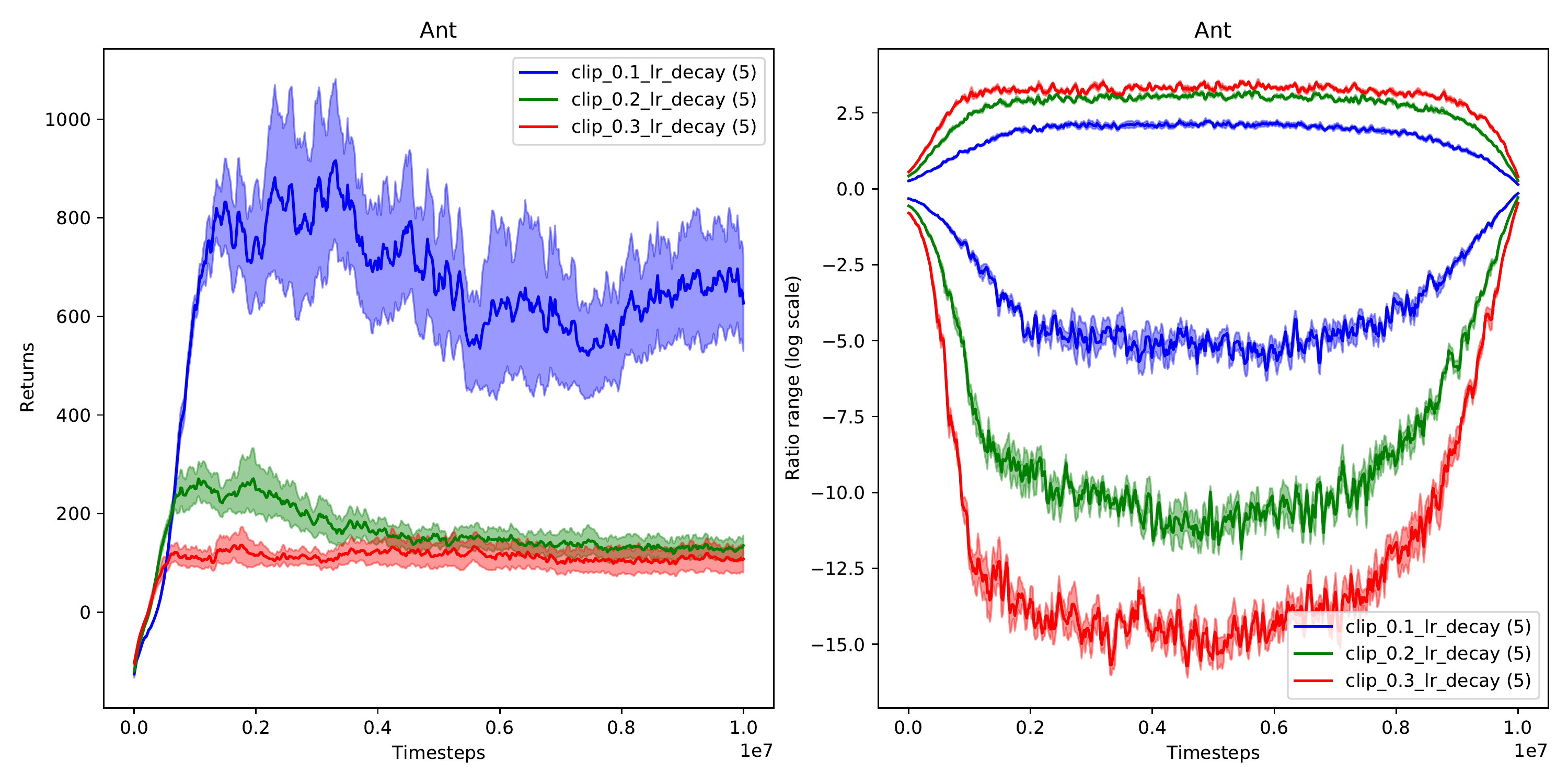}
    \caption{Empirical returns and ratio ranges (log scale) of RC-PPO trained on Ant-v2 with different clipping values; the learning rate decays linearly from $0.0003$ to $0$. }
    \label{fig:ant-ratio-range}
\end{figure}

Though this ratio bounding interpretation has been widely discussed~\citep{akrour2019projections,queeney2021generalized}, 
we find that it may not be true in practice:
our empirical results show that the ratios can easily depart from the range $[0.5, 2.0]$, 
which can make the bound on the maximum TV divergence too loose to be nontrivial according to the above corollary. 
Figure~\ref{fig:ant-ratio-range-no-decay} shows the training curves of ratio clipping PPO on Mujoco Ant domain, 
with different hyperparameters for the clipping and a fixed learning rate $0.0003$. 
It suggests that the clipping range has a big impact on training performance.
However, clipping alone (without any learning rate annealing scheme) fails to bound the ratios. 
Even with a proper annealing scheme, e.g., a linear learning rate decay as used in Figure~\ref{fig:ant-ratio-range},
the ratios can still go beyond the scope of $[0.5, 2.0]$. 
For example, for clipping range $0.1$ with which PPO achieves the best performance,
the maximum ratio reaches $e^{1.5}\approx 4.482$ and the minimum ratio $e^{-5.0}\approx 0.007$.
This ratio range imposes only a trivial bound on the maximum TV divergence. 
Furthermore, the ratio range is also highly sensitive to the number of optimization epochs performed in each iteration. 
Specifically, as the number of optimization epochs increases, 
the ratios grow without bound even with the same set of the clipping value and the learning annealing scheme. 
Refer to Figure~\ref{fig:ant-ratio-range-epoch} in Appendix~\ref{app:more-results} for more details. 
\begin{remark}
    In practice, the ratios in PPO are not bounded such that the resulting maximum TV divergence is nontrivially bounded. 
    Furthermore, ratio clipping itself is not sufficient for strong performance. 
\end{remark}

\section{Ratio-regularized policy optimization}
Before introducing our new policy optimization algorithm, 
we first revisit the underlying idea in PPO. 
As pointed out by~\citet{schulman2017proximal}, 
PPO reuses samples for multiple epochs of policy optimization. 
This reuse of samples, however, does not come for free. 
Theorem~\ref{theo:trpo-theorem} suggests that,
whenever the policy is updated, it deviates from the policy used to collect samples, i.e., behaviour policy, and thus
reusing those samples for the next optimization epoch yields a weaker lower bound for policy improvement. 
If one proceeds with multiple epochs anyway, 
the lower bound in Theorem~\ref{theo:trpo-theorem} would not guarantee any improvement of the policy~\citep{kakade2002approximately,pirotta2013safe}.  
This inspires us to develop an algorithm that finds the proper time to stop sample reuse, 
before the lower bound becomes trivial. 
However, the performance lower bound presented in Theorem~\ref{theo:trpo-theorem} relies on the maximum TV divergence, 
which is intractable to estimate in practice. 
In the following, we present a new monotonic improvement guarantee with a more practical lower bound. 

\subsection{Ratio-regularized improvement guarantee}
We present the ratio-regularized improvement guarantee. 
\begin{theorem}\label{theo:ratio-reg-guarantee}
For any policies $\tilde{\pi}$ and $\pi$, 
the following bound holds:
\begin{multline}\label{equ:ratio-monotonic-bound}
J(\tilde{\pi}) - J(\pi) \geq \frac{1}{1-\gamma}\Big\{ \mathbb{E}_{(s, a)\sim d_\pi}\Big[ \frac{\tilde{\pi}(a|s)}{\pi(a|s)} A_{\pi}(s, a) \Big] \\
- C \cdot \mathbb{E}_{(s, a)\sim d_{\pi} }\left|\frac{\tilde{\pi}(a|s)}{\pi(a|s)} - 1 \right| \Big\}. 
\end{multline}
where $C=\frac{\xi \gamma}{1-\gamma}$, $\xi=\max_{s,a}\lvert A_{\pi}(s, a)\rvert$.
\end{theorem}
The proof is given in Appendix~\ref{app:ratio-reg-guarantee}. 
This theorem implies that the performance gap between any two policies 
can be effectively bounded by the expected absolute ratio deviations
$\mathbb{E}_{(s, a)\sim d_{\pi}}\left|\frac{\tilde{\pi}(a|s)}{\pi(a|s)} - 1 \right|$. 

It differs from Theorem~\ref{theo:trpo-theorem} in two respects. 
First, it indicates that the lower bound can be 
effectively achieved by averaging the absolute ratio deviations, 
i.e., absolute ratio deviation in expectation. 
Compared to the maximum TV divergence adopted in Theorem~\ref{theo:trpo-theorem}, 
this absolute ratio deviation in expectation can be more feasibly estimated in practice:
one just needs to sample a batch of samples and compute the average absolute deviation. 
Second, this theorem pivots on the ratio $\frac{\tilde{\pi}(a|s)}{\pi(a|s)}$ 
and one can therefore formulate a ratio regularized policy optimization, 
just like KL-regularized PPO. 
For completeness, we also present this type of regularization formulation and its empirical performance in Appendix~\ref{app:r2po}.

\subsection{Early Stopping Policy Optimization (ESPO)}
Similar to PPO, we are interested in optimizing a policy through multiple epochs by reusing samples. 
In practice, if we use the regularization coefficient $C$ recommended by Theorem~\ref{theo:ratio-reg-guarantee}, 
the step sizes for policy optimization would be too small. 
To mitigate this, we adopt the same strategy as in TRPO~\citep{schulman2015trust} to use
a constraint on the ratio deviations:
\begin{align}\label{equ:espo-formulation}
\max_{\tilde{\pi}} \quad & \mathbb{E}_{(s, a)\sim d_\pi}\Big[ \frac{\tilde{\pi}(a|s)}{\pi(a|s)} A_{\pi}(s, a) \Big], \nonumber \\
\text{s.t.} \quad & \mathbb{E}_{(s, a)\sim d_{\pi} }\left|\frac{\tilde{\pi}(a|s)}{\pi(a|s)} - 1 \right| \leq \delta. 
\end{align}
As a result, we can consider optimizing the surrogate objective 
with multiple optimization epochs without any ratio clipping
and meanwhile monitor the expected absolute ratio deviations. 
In particular, after each epoch of optimizing the surrogate objective, 
we estimate the ratio deviations and, 
if they exceed a threshold, we early stop from the optimization epochs. 
We call the resulting algorithm Early Stop Policy Optimization (ESPO), 
as detailed in Algorithm~\ref{algo:ESPO}. 

\begin{algorithm}
    \caption{Early Stopping Policy Optimization (ESPO)}
    \label{algo:ESPO}
    \begin{algorithmic}
        \FOR{iterations $i=1, 2,...$}
            \FOR{actor $=1, 2,...,N$}
                \STATE Run policy $\pi$ in environment
                \STATE Compute advantage estimates $\hat{A}_{\pi}$
            \ENDFOR
            \FOR{epoch $=1, 2,...,K$}
                \STATE Sample $M$ samples $\{(s, a)\}$ from previous rollouts.
                \STATE Compute $\mathcal{L}(\theta) \triangleq - \frac{1}{M}\sum_{s, a}\frac{\tilde{\pi}_{\theta}(a|s)}{\pi(a|s)} \hat{A}_{\pi}(s, a)$.
                \STATE Optimize loss $\mathcal{L}(\theta)$ w.r.t $\theta$. 
                \STATE Estimate deviations $\Delta \triangleq \frac{1}{M}\sum_{s, a}\left|\frac{\tilde{\pi}_{\theta}(a|s)}{\pi(a|s)} - 1 \right|$. 
                \IF{$\Delta > \delta $}
                \STATE Break and early stop from optimization epochs. 
                \ENDIF
            \ENDFOR
            \STATE $\pi \leftarrow \tilde{\pi}_{\theta}$. 
        \ENDFOR
    \end{algorithmic}
\end{algorithm}

ESPO uses the surrogate objective in RC-PPO without any ratio clipping. 
The major difference between these two algorithms is that 
ESPO uses an indicator function to estimate the proper time to drop out from the multiple optimization epochs
while PPO considers a fixed number of optimization epochs. 
Moreover, unlike RC-PPO, which requires tuning the clipping range and the right number of optimization epochs,
ESPO only needs one hyperparameter, the stopping threshold $\delta$. 
Theorem~\ref{theo:ratio-reg-guarantee} suggests that $\delta$ can be heuristically set to a fixed value 
as the coefficient $C$ can be a constant if the advantage is normalized before the iteration begins.
In the experiment section, we show that this hyperparameter can be tuned in one domain 
and generalizes well to several other domains.

Furthermore, one can use  Pinsker's and Jensen's inequalities to derive an upper bound for the ratio deviations. 
Specifically, according to Corollary 3 from~\citet{achiam2017constrained}, 
we can have the following corollary.
\begin{corollary}\label{coro:rd-kl-bound}
    We have the following bound:
    \begin{equation*}
        \Big[\mathbb{E}_{(s, a)\sim d_{\pi}}\left| \frac{\tilde{\pi}(a|s)}{\pi(a|s)} -1  \right| \Big]^2 \leq 2 \mathbb{E}_{(s,a)\sim d_{\pi}}\log\big[\frac{\pi(a|s)}{\tilde{\pi}(a|s)} \big].
    \end{equation*}
    The equality holds if and only if $\tilde{\pi}(a|s) = \pi(a|s)$ for all $(s, a)\in\mathcal{S}\times\mathcal{A}$.
\end{corollary}
One can thus reformulate~\eqref{equ:espo-formulation} as follows:
\begin{align}\label{equ:kl-es-formulation}
\max_{\tilde{\pi}}& \quad \mathbb{E}_{(s, a)\sim d_{\pi}}\Big[\frac{\tilde{\pi}(a|s)}{\pi(a|s)}A_{\pi}(s, a) \Big], \nonumber \\
\text{s.t.}& \quad \mathbb{E}_{(s, a)\sim d_{\pi}} \Big[ \log\frac{\pi(a | s)}{\tilde{\pi}(a|s)} \Big] \leq \beta, 
\end{align}
where $\beta$ is a hyperparameter set according to $\delta$. 
This is also aligned with the practical formulation of TRPO in~\eqref{equ:trpo-objective}
(or the KL-PPO formulation in~\eqref{equ:kl-ppo-objective}). 
The resulting algorithm is then to instead use $\mathbb{E}_{(s, a)\sim d_{\pi}} \Big[ \log\frac{\pi(a | s)}{\tilde{\pi}(a|s)} \Big]$
to decide the early stopping.
We call this mechanism of early stopping in~\eqref{equ:kl-es-formulation} KL Early Stopping (KL-ES) 
as it relates to the KL divergence,  
and the preceding one in~\eqref{equ:espo-formulation} Ratio Deviation (RD) Early Stopping (RD-ES). 

Furthermore, Corollary~\ref{coro:rd-kl-bound} implies that 
for any $\delta$ for KL-ES, one can always find a corresponding $\delta$ for RD-ES, 
and RD can always be well bounded by KL. 
In the experiment section, we show that RD-ES performs better than KL-ES in many domains and 
is less sensitive to the threshold values. 

\begin{figure*}
    \centering
    \includegraphics[width=\linewidth]{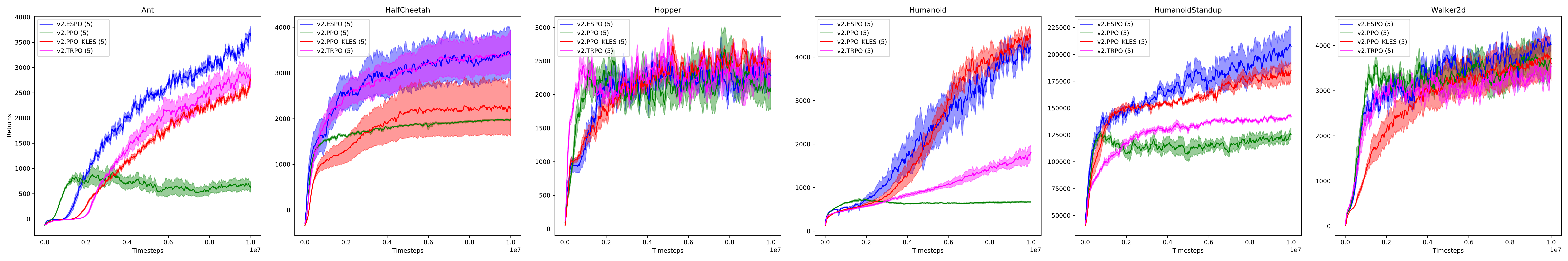}
    \caption{ESPO performance against ratio clipping PPO, ratio clipping PPO with KL to early stop (PPO-KLES), and TRPO on Mujoc benchmark tasks. }
    \label{fig:mujoco-comparison}
\end{figure*}

\subsection{Distributed ESPO}
RC-PPO is also known for its ability to scale up in distributed settings~\citep{berner2019dota,ye2020mastering}. 
In this section, we show that ESPO can also be easily scaled up. 
Similar to Distributed PPO in~\citep{heess2017emergence}, 
we distribute the data collection and gradient computation over workers. 
During training, gradients are first averaged across all workers 
and then synchronised to each worker. 
We adopt this distributed training paradigm as it is reported to yield strong empirical results in practice~\citep{heess2017emergence}.
Likewise, we compute the ratio deviations for each worker 
and use the \emph{minimum} value of ratio deviations across all workers, i.e., \emph{reduce-min}, 
to decide whether to early stop from the optimization epochs. 
Such early stopping is synchronised over all workers.  
We have experimented with other ways of reducing ratio deviations across all workers 
and have found that synchronously reducing to the minimum leads to better results in practice. 
Refer to Algorithm~\ref{algo:despo} in Appendix~\ref{app:distributed-espo}.

\section{Experiments}

For our experiments we first compare ESPO with RC-PPO and other baselines across a wide range of tasks. 
Specifically, we first consider continuous control Mujoco benchmarks~\citep{brockman2016openai}, 
and then scale up to higher dimension continuous control in the DeepMind Control Suite~\citep{tassa2018deepmind}. 
Table~\ref{tab:control-complexity} in Appendix~\ref{app:implementations} gives an overview of the control complexity in terms of state-action dimensions. 
In all tasks we use a Gaussian distribution for the policy whose mean and covariance are parameterized by a neural network 
(see Appendix~\ref{app:implementations} Table~\ref{tab:hyperparameters} for more details). 
Second, we scale up the training to many workers 
and compare the performance of Distributed ESPO with Distributed RC-PPO. 
We set the stopping threshold in ESPO to $0.25$ 
as we found it generalizes well to all control tasks. 
Following~\citet{schulman2017proximal,heess2017emergence}, 
we also normalize observations, rewards and advantages for ESPO and Distributed ESPO. 
See Figure~\ref{fig:ablate-norm-full} in Appendix~\ref{app:more-results} for the ablations. 
The implementation details can also be found in Appendix~\ref{app:implementations}. 
All the empirical results are reported by the mean and standard deviation, 
with 5 random repeated runs (as indicated in the figure legend).

\subsection{High dimensional control tasks}

\begin{figure}
    \centering
    \includegraphics[width=0.9\linewidth]{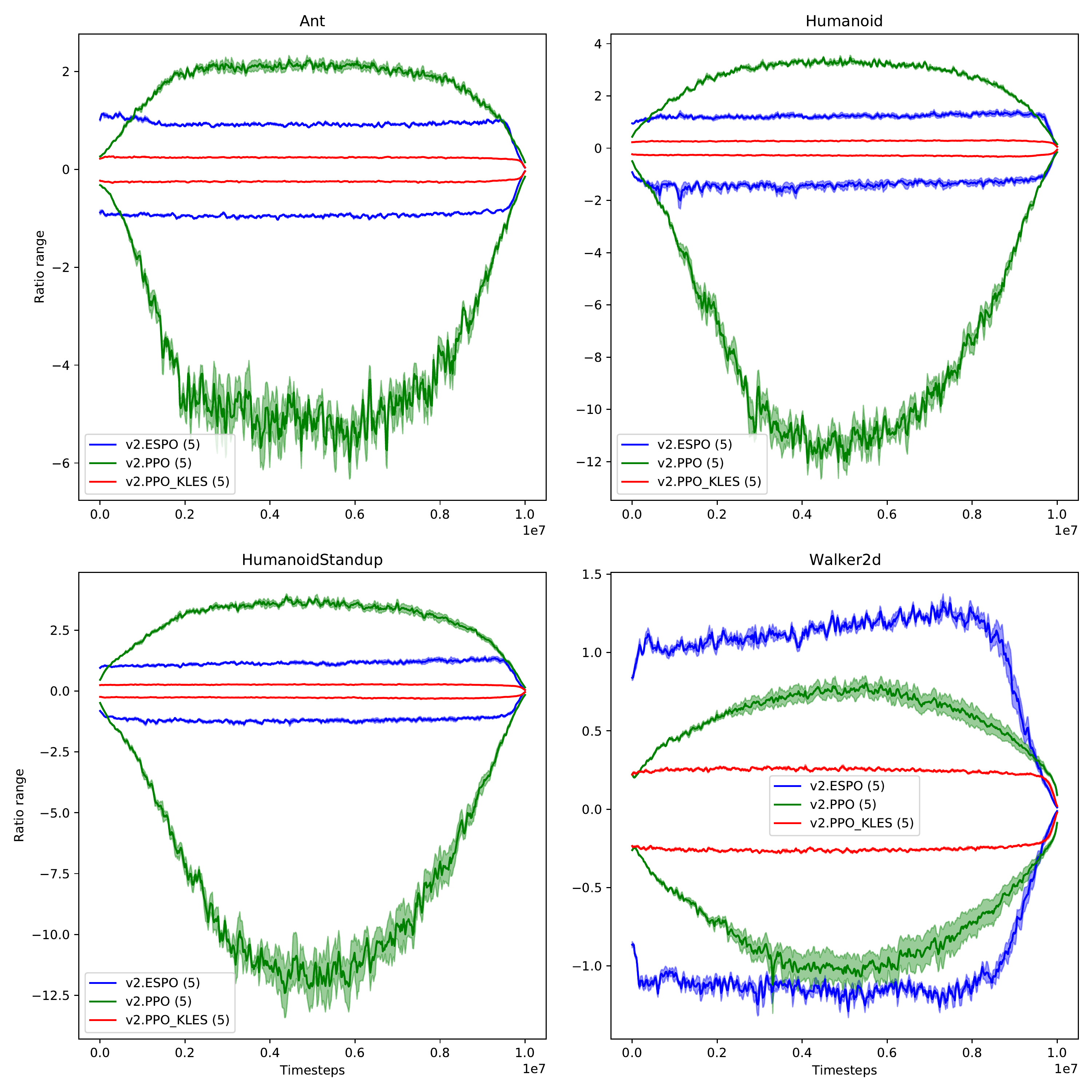}
    \caption{Ratio ranges for ratio clipping PPO, PPO-KLES and ESPO on four Mujoco tasks.}
    \label{fig:mujoco-ratios}
\end{figure}

\begin{figure*}
    \centering
    \includegraphics[width=0.7\linewidth]{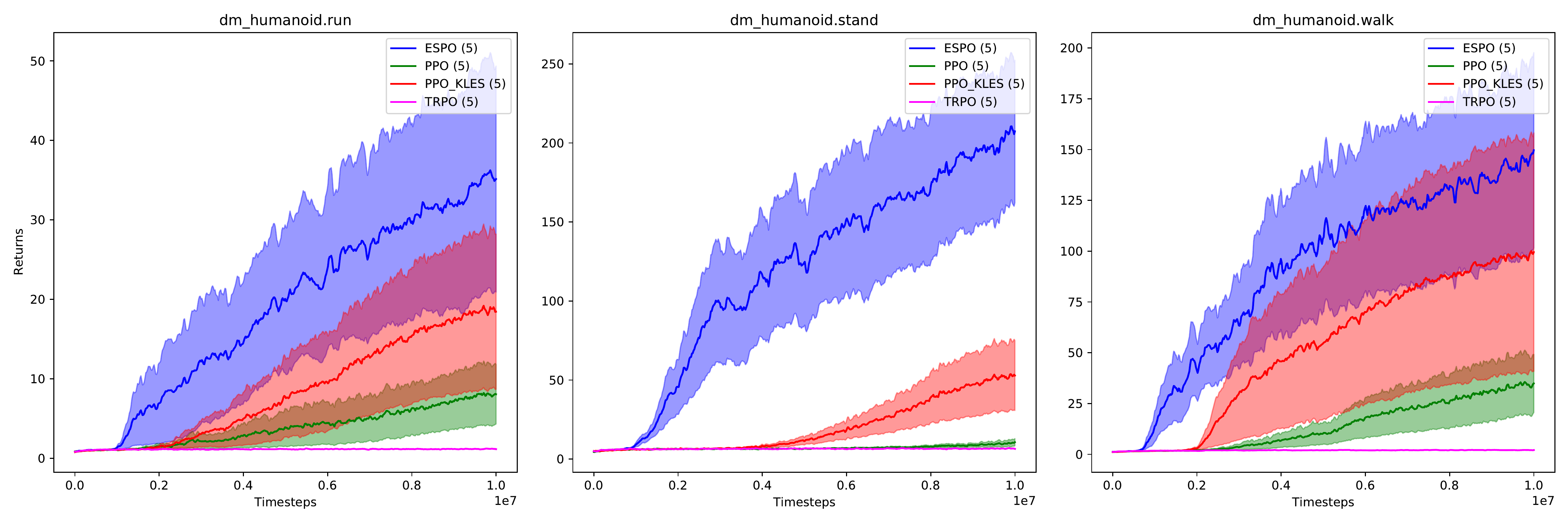}
    \caption{ESPO performance against ratio clipping PPO, PPO-KLES and TRPO on high dimensional DeepMind control tasks (i.e., humanoid run, stand and walk).}
    \label{fig:dmc-comparison}
\end{figure*}

\begin{figure*}
    \centering
    \includegraphics[width=\linewidth]{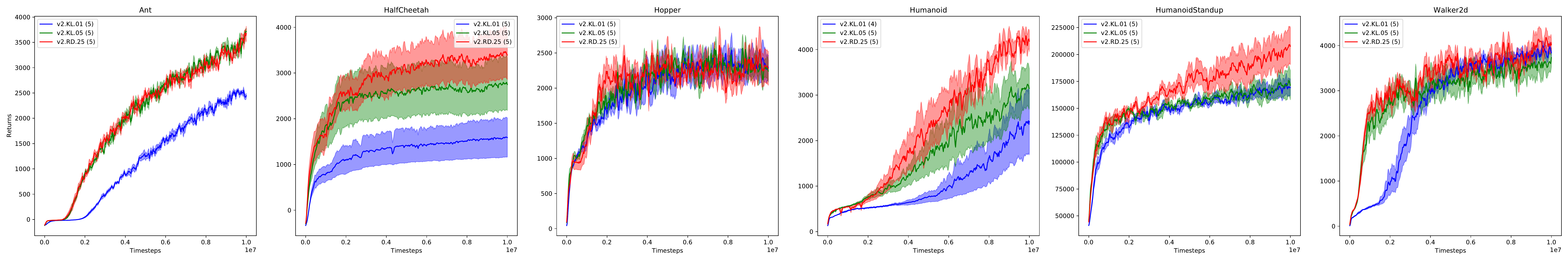}
    \caption{Different ways to early stop from the optimization epochs: 
        \emph{KL.05} and \emph{KL.10} correspond to using the KL estimate as the early stopping condition 
        (with $\beta$ equals $0.05$ and $0.10$ respectively as in~\eqref{equ:kl-es-formulation});
        \emph{RD.30} refers to using the ratio deviations as the early stopping condition with $\delta=0.30$ as in~\eqref{equ:espo-formulation}.}
    \label{fig:ablate-rd-vs-kl}
\end{figure*}

We first compare ESPO with RC-PPO and TRPO on high dimensional control tasks. 
For RC-PPO, we use the default hyperparameters used by~\citet{schulman2017proximal}. 
We also sweep over many different values for the clipping range 
and pick the one that performs the best as the baseline. 
Moreover, we train RC-PPO  with proper learning rate annealing 
(as it would fail to optimize policies with a fixed learning rate; 
more results can be found in Appendix~\ref{app:more-results}.)
TRPO was also trained with default parameters in~\citet{schulman2015trust}.
Furthermore, we noticed that the OpenAI spinning up implementation of RC-PPO heuristically uses the KL estimate to early stop the optimization epochs, 
as presented in~\eqref{equ:kl-es-formulation}. 
We also leverage this version of RC-PPO as an important baseline, i.e., PPO-KLES 
and adopt the best reported threshold value, $0.01$, for this KL-ES. 
The performance comparison is presented in Figure~\ref{fig:mujoco-comparison}.

ESPO outperforms PPO and TRPO by a largin margin on Ant, Humanoid, HumanoidStandup and Walker2d. 
In these three domains, the performance of PPO plateaus quickly after 1 million timesteps 
while the performance of ESPO is still monotonically increasing. 
Moreover, comparing the training performance of PPO-KLES and PPO, 
shows that early stopping of the optimization epochs via an empirical KL estimate 
has a big impact on PPO performance. 
In particular, the KL early stopping in PPO significantly boosts the performance of 
the vanila ratio clipping PPO on Ant, HalfCheetah, Humanoid and HumanoidStandup,
which suggests that the early stopping from the optimization epochs can be crucial. 
Also, ESPO delivers more consistent good performs across all Mujoco domains than PPO-KLES. 
Figure~\ref{fig:mujoco-ratios} presents the ratio range for ESPO, PPO and PPO-KLES in the training. 
The ratios in ESPO have a similar evolving pattern as that in TRPO 
while those in PPO can grow unbounded. 
See Figure~\ref{fig:mujoco-ratios-full} in Appendix~\ref{app:more-results} for a full report. 
Notably, the ratios in PPO are not well bounded in Ant, Humanoid and HumanoidStandup, 
which is consistent with its poor performance on these tasks. 
With early stopping, either via ratio deviations as used in ESPO or empirical KL as used in PPO-KLES, 
the ratios in all four tasks are well bounded. 
Moreover, we also report the empirical estimate of ratio deviations in Figure~\ref{fig:mujoco-empirical-tv} in Appendix~\ref{app:more-results}. 
Overall, compared to the RC-PPO, 
introducing early stopping in both ESPO and PPO-KLES significantly reduces the empirical ratio deviations, 
which is well aligned with the formulation in~\eqref{equ:espo-formulation} and~\eqref{equ:kl-es-formulation}.

We further compare ESPO with RC-PPO, PPO-KLES and TRPO on more complex control domains: 
 the humanoid tasks in the DeepMind Control (DMC) Suite~\citep{tassa2018deepmind}, 
including Humanoid run, stand, and walk.
The empirical results are presented in Figure~\ref{fig:dmc-comparison}. 
ESPO significantly outperforms all other baselines in all those highly complex domains.
The results also show that, 
unlike all other multi-epoch policy optimization methods, 
TRPO completely fails to effectively optimize policies. 
The empirical ratio ranges and the estimate of ratio deviations
are also presented in Figure~\ref{fig:dm-ratios-full} and Figure~\ref{fig:dm-tv-full} repsectively 
in Appendix~\ref{app:more-results}. 
Compared to PPO and PPO-KLES, 
ESPO is more effective in bounding the ratio deviation in all three sub-tasks 
and maintaining it at an almost constant level, 
which could be more benificial for the policy to take large update steps
without breaking the monotonic improvement guarantee.

\subsection{Comparing different ways of early stopping}

We compare the two different ways of early stopping, KL-ES and RD-ES. 
Specifically, we perform a hyperparameter sweep for the thresholding value used in KL-ES, 
and find that thresholding value $0.05$ generally performs the best, which we denote KL.05.
We also compute another threshold value $0.01$ for KL-ES 
according to the SpiningUp implementation, 
which produces the best performance with RC-PPO, 
and denote this baseline KL.01. 
For RD-ES, we use the threshold value $0.25$ as it generally works well across all Mujoco domains, 
and denote it RD.25.
The results are presented in Figure~\ref{fig:ablate-rd-vs-kl}. 
RD-ES consistently outperforms the two variants of KL-ES, 
and the performance of RD.25 is more robust than its KL counterpart. 
This could be because the RD estimate for early stopping 
is always tighter than that of KL divergence 
since the ratio deviation is upper bounded by the KL divergence based on Corollary~\ref{coro:rd-kl-bound}. 
The ratio range and the empirical estimate of expected ratio deviations are also given in Figure~\ref{fig:rd-vs-kl-ratios} and Figure~\ref{fig:rd-vs-kl-tv} in Appendix~\ref{app:more-results}. 
Interestingly, although the ratios are bounded very similarly in RD.25 and KL.05, 
their corresponding expected ratio deviations differ greatly:
the ratio deviation in KL.05 is much smaller than that in RD.25. 
Since the ratio deviation describes the policy updates at each state-action samples, 
i.e., how much the updated policy depart from the previous policy at each empirical sample, 
this behavior difference between KL.05 and RD.25 possibly implies that 
RD.25 and KL.05 updates the policy differently. 
Specifically, RD.25 optimizes the policy to be more uniformly deviated from the old policy than KL.05, 
which could make the policy optimization more effectively in practice.

\begin{figure}
    \centering
    \includegraphics[width=1.0\linewidth]{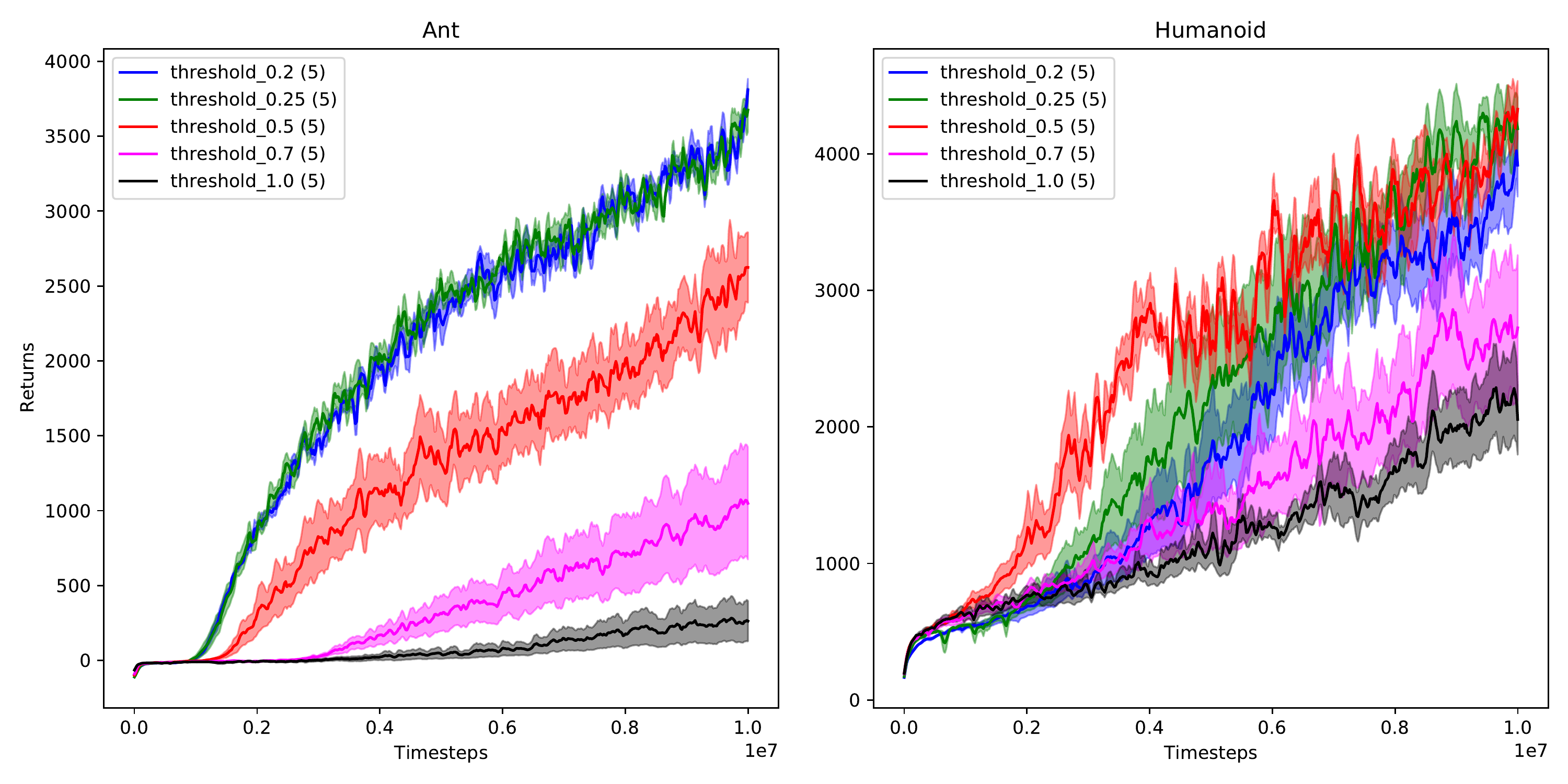}
    \caption{Sensitivity analysis on the threshold values, i.e., when to stop the multiple optimization epochs.}
    \label{fig:ablate-threshold}
\end{figure}

In ESPO, one important hyperparameter is the stopping threshold, 
which indicates when to stop the multiple optimization epochs.
We perform a sensitivity analysis on different values for this threshold value, 
as presented in Figure~\ref{fig:ablate-threshold}.
The threshold value is crucial for ESPO, 
which is consistent with Theorem~\ref{theo:ratio-reg-guarantee} 
as it trades off between a tight lower bound and sample efficiency. 
However, ESPO with threshold $0.25$ performs generally well on different domains. 
See 
Appendix~\ref{app:more-results} for the full ablation results.

\subsection{Comparing to distributed PPO with many workers}
PPO is also known for its ability to scale up in distributed settings~\citep{schulman2017proximal,heess2017emergence,berner2019dota,ye2020mastering}.
We thus also compare ESPO to Distributed PPO (DPPO)~\citep{heess2017emergence}. 
Specifically, we train ESPO in a single process for 50 million timesteps 
and compare this single-process ESPO, i.e., ESPO-1-worker, with PPO trained with multiple processes. 
Similar to~\citet{heess2017emergence}, we run 10 workers in parallel, i.e., PPO-10-workers, 
each of which is assigned one actor to collect many timesteps of data 
and one learner to compute the gradient and update policy parameters~\citep{espeholt2018impala,ye2020mastering}. 
The gradients of the PPO learners are averaged through the Message Passing Interface (MPI) all-reduce algorithm~\citep{sergeev2018horovod}. 
We use the more complicated Mujoco tasks, 
including Ant, Humanoid, HumanoidStandup and Walker2d. 
We also scale up the training of ESPO in the same manner according to Algorithm~\ref{algo:despo}, 
and denote its performance as ESPO-10-workers.
The training performance of ESPO and DPPO is presented in Figure~\ref{fig:espo-vs-dppo}. 
We also report the single process PPO, PP0-1-worker, for reference.
Clearly, scaling up the training via multiple workers 
significantly improves PPO's performance: 
PPO-10-workers significantly outperforms its single process version,
on Ant, Humanoid, and HumanoidStandup. 
Hence, large-scale training may explain the success of PPO in many game applications~\citep{berner2019dota,ye2020mastering},
Also, with the same number of samples, 
ESPO still outperforms DPPO on Humanoid and HumanoidStandup, 
even performs better than its scaled-up variant, ESPO-10-workers.
On Ant domain, the performance of ESPO is worse than that of DPPO. 
But the distributed ESPO, i.e., ESPO-10-worker, performs rather similarly as DPPO
It could be that, on this specific domain, 
averaging the gradient across many workers reduces the variance 
and makes policy optimization more effective in each iteration.

\section{Conclusion}
We investigated Proximal Policy Optimization (PPO) methods 
and considered one of its most important variants, ratio clipping PPO. 
One distinctive feature of this method is performing 
multiple optimization epochs of a surrogate objective with one set of sampled data. 
We showed that ratio clipping is not needed:
one can directly optimize the original surrogate objective for multiple epochs without any clipping. The key is to early stop such optimization epochs at the right time. 
We presented theoretical analysis to shed light on how to decide this early stopping time
and proposed a novel algorithm, Early Stopping Policy Optimization (ESPO). 
We compared ESPO with PPO across many continuous control tasks 
and showed that ESPO is significantly better than PPO in those benchmark tasks. 
We also showed that ESPO can be easily adapted to distributed training with many workers, 
and still delivers strong empirical performance, 
comparable to or better than the PPO counterpart. 

\bibliography{bib}
\bibliographystyle{icml2022}

\newpage
\appendix
\onecolumn
\section{Theorem proofs}

\subsection{Proof of Proposition~\ref{prop:ratio-tv-bound}}\label{app:ratio-tv-bound}
\begin{proposition}\label{prop:tv-max-identity}
$\TVmax(\pi, \tilde{\pi}) = \max_{s\in\mathcal{S}} \displaystyle\sum_{\tilde{\pi}(a|s)\geq\pi(a|s)}\big[ \tilde{\pi}(a|s) - \pi(a|s) \big] $. 
\end{proposition}
This useful identity follows from a property of $\TV$: 
$\TV(\mu(x), \nu(x))=\sum_{\mu(x)>\nu(x)}[\mu(x)-\nu(x)]$ 
where $\mu$ and $\nu$ are two distributions. 
This proposition indicates that a decentralized trust region is also defined by
the sum of probability differences over a special subset. 
We use this to upper bound the trust region in the following analysis.

\begin{assumption}
Assume the advantage function $A_{\pi}(s, a)$ converges to a fixed point.
\end{assumption}

\begin{proposition}
If $\lambda(s, a) \triangleq \frac{\tilde{\pi}(a|s)}{\pi(a|s)}$ is bounded by $M_1$ and $M_2$, 
i.e., $M_1 \leq \lambda(s, a) \leq M_2$ for $\forall (s, a)$ 
where $M_1\in(0, 1)$ and $M_2\in(1,\infty)$, 
then the following bound holds:
$\TVmax\big(\pi(\cdot|s), \tilde{\pi}(\cdot|s)\big) \leq \min(\frac{1}{M_1} -1 , M_2 -1)$. 
\end{proposition}
\begin{proof}
This proposition comes from that fact that 
updating $\tilde{\pi}(s, a)$ with respect to a converged $A(s, a)$
leads to $\tilde{\pi}(a|s)>\pi(a|s)$ if $ A(s, a)>0,~\forall s, a$ in actor-critic algorithms~\citep{konda2000actor}. 
Thus, based on Proposition~\ref{prop:tv-max-identity}, we have the following
\begin{equation*}
\TVmax(\pi, \tilde{\pi}) = \max_{s}\sum_{\substack{a\in\mathcal{A}\\ A(s,a)>0}}[\tilde{\pi}(a|s) - \pi(a|s)]
\leq \max_{s}\sum_{\substack{a\in\mathcal{A} \\A(s,a)>0}}[M_2 \pi(a|s) - \pi(a|s)]
\leq M_2 - 1. 
\end{equation*}
The equation is from Proposition~\ref{prop:tv-max-identity} by considering $A(s, a)>0$
such that $\tilde{\pi}(a|s)>\pi(a|s)$. 
The first inequality is a result of bounded ratios 
and the second is from $\sum_{a\in\{a: A(s, a)>0\}}[\pi(a|s)]<1$ 
(considering the lower bound yields the same analysis.)
Likewise, we have the following
\begin{equation*}
\TVmax(\pi, \tilde{\pi}) = \max_{s}\sum_{\substack{a\in\mathcal{A}\\ A(s,a)<0}}[\pi(a|s) - \tilde{\pi}(a|s)] 
\leq \max_{s}\sum_{\substack{a\in\mathcal{A} \\A(s,a)<0}}[\frac{1}{M_1}\tilde{\pi}(a|s) - \tilde{\pi}(a|s)]
\leq \frac{1}{M_1} - 1. 
\end{equation*}
Combined, $\TVmax(\pi, \tilde{\pi}) \leq \min(\frac{1}{M_1} - 1, M_2-1)$. 
\end{proof}

\subsection{Proof of Theorem~\ref{theo:ratio-reg-guarantee}}\label{app:ratio-reg-guarantee}
\begin{theorem}
For any policies $\tilde{\pi}$ and $\pi$, 
the following bound holds:
\begin{equation*}
J(\tilde{\pi}) - J(\pi) \geq \frac{1}{1-\gamma}\bigg\{ \mathbb{E}_{(s, a)\sim d_\pi}\Big[ \frac{\tilde{\pi}(a|s)}{\pi(a|s)} A_{\pi}(s, a) \Big]
- C \cdot \mathbb{E}_{(s, a)\sim d_{\pi} }\left|\frac{\tilde{\pi}(a|s)}{\pi(a|s)} - 1 \right| \bigg\}. 
\end{equation*}
where $C=\frac{\xi \gamma}{1-\gamma}$ and $\xi=\max_{s,a}\lvert A_{\pi}(s, a)\rvert$.
\end{theorem}
\begin{proof}
This proof relies on the same strategy used in~\citep{schulman2015trust,achiam2017constrained}. 
According to~\eqref{equ:performance-identity},
for any two policies $\tilde{\pi}$ and $\pi$, we have the following
\begin{align*}
& J(\tilde{\pi}) - J(\pi) = \frac{1}{1-\gamma} \mathbb{E}_{\red{s\sim d_{\tilde{\pi}}}, a\sim\tilde{\pi}}[A_{\pi}(s, a)] \\
 = & \frac{1}{1-\gamma} \Big[\mathbb{E}_{\red{s\sim d_{\tilde{\pi}}}, a\sim\tilde{\pi}}[A_{\pi}(s, a)] - \mathbb{E}_{\blue{s\sim d_{\pi}}, a\sim\tilde{\pi}}[A_{\pi}(s, a)]
 + \mathbb{E}_{\blue{s\sim d_{\pi}}, a\sim\tilde{\pi}}[A_{\pi}(s, a)] \Big] \\
 = & \frac{1}{1-\gamma} \Big[ \underbrace{\mathbb{E}_{\red{s\sim d_{\tilde{\pi}}}, a\sim\tilde{\pi}}[A_{\pi}(s, a)] - \mathbb{E}_{\blue{s\sim d_{\pi}}, a\sim\tilde{\pi}}[A_{\pi}(s, a)]}_{\text{Correction term}} 
 + \underbrace{\mathbb{E}_{\blue{s\sim d_{\pi}}, a\sim\pi}\big[ \frac{\tilde{\pi}(a|s)}{\pi(a|s)} A_{\pi}(s, a) \big]}_{\text{Surrogate term}} \Big]. 
\end{align*}
In the last transition we unify the terms under the same expectation sign by using importance weights. 
Consider the correction term
\begin{equation*}
\mathbb{E}_{\red{s\sim d_{\tilde{\pi}}}, a\sim\tilde{\pi}}[A_{\pi}(s, a)] - \mathbb{E}_{\blue{s\sim d_{\pi}}, a\sim\tilde{\pi}}[A_{\pi}(s, a)]
= \sum_{s, a} \big[\red{d_{\tilde{\pi}}(s)} - \blue{d_{\pi}(s)}\big]\tilde{\pi}(a|s)A_{\pi}(s, a)
= \sum_{s} \big[\red{d_{\tilde{\pi}}(s)} - \blue{d_{\pi}(s)}\big] \sum_{a}\big[ \tilde{\pi}(a|s)A_{\pi}(s, a) \big]. 
\end{equation*}
Using vector notation $\red{d_{\tilde{\pi}}} - \blue{d_{\pi}}$ and $A_{\pi}^{\tilde{\pi}}$, 
we can rewrite the above summation in dot product form:
\begin{equation*}
\mathbb{E}_{\red{s\sim d_{\tilde{\pi}}}, a\sim\tilde{\pi}}[A_{\pi}(s, a)] - \mathbb{E}_{\blue{s\sim d_{\pi}}, a\sim\tilde{\pi}}[A_{\pi}(s, a)]= (\red{d_{\tilde{\pi}}} - \blue{d_{\pi}})\cdot A_{\pi}^{\tilde{\pi}}
\end{equation*}
This term can be bounded by applying Holder's inequality: for any $p, q\in[1, \infty]$, 
such that $\frac{1}{p} + \frac{1}{q} = 1$, 
we have
\begin{equation*}
\norm{(\red{d_{\tilde{\pi}}} - \blue{d_{\pi}})\cdot A_{\pi}^{\tilde{\pi}} }_1 \leq \norm{\red{d_{\tilde{\pi}}} - \blue{d_{\pi}}}_p \norm{A_{\pi}^{\tilde{\pi}}}_q. 
\end{equation*}

We consider the case $p=1$ and $q=\infty$ as in~\citep{schulman2015trust} and ~\citep{achiam2017constrained}, 
and aim at bounding $ \norm{\red{d_{\tilde{\pi}}} - \blue{d_{\pi}}}_1$ and $\norm{A_{\pi}^{\tilde{\pi}}}_\infty$. 

First, we show how to bound $\norm{\red{d_{\tilde{\pi}}} - \blue{d_{\pi}}}_1$.
Let $G=(\mathbf{1} +\gamma P_\pi + (\gamma P_\pi)^2 + ...) = (\mathbf{1}-\gamma P_\pi)^{-1}$ (a direct result of series summation formula), 
and similarly let $\tilde{G}=(\mathbf{1} + \gamma P_{\tilde{\pi}} + (\gamma P_{\tilde{\pi}})^2 + ...) = (\mathbf{1}-\gamma P_{\tilde{\pi}})^{-1}$. 
\begin{equation*}
    G^{-1} - \tilde{G}^{-1} = (\mathbf{1}-\gamma P_{\pi}) - (\mathbf{1}-\gamma P_{\tilde{\pi}}) = \gamma \Delta,
\end{equation*}
where $\Delta\triangleq P_{\tilde{\pi}} - P_{\pi}$. 
Right multiply by $G$ and left multiply by $\tilde{G}$ and rearrange:
$\tilde{G} - G  = \gamma \tilde{G}\Delta G$. 
Thus, 
\begin{equation*}
\red{d_{\tilde{\pi}}} - \blue{d_{\pi}} 
= (1-\gamma)(\tilde{G} - G)p_0 = (1-\gamma) \gamma \tilde{G}\Delta G p_0 
= \gamma \tilde{G}\Delta \blue{d_{\pi}}. 
\end{equation*}
The last transition comes from the definition of $d_\pi$: $d_\pi \triangleq Gp_0$. 
According to the $l_1$ operator norm
$\norm{A}_1 = \sup_{ d} \left\{ \frac{\norm{A d}_1}{\norm{ d}_1} \right\}$, 
we have
\begin{align*}
\norm{\red{d_{\tilde{\pi}}} - \blue{d_{\pi}}}_1 &= \gamma \norm{ \tilde{G}\Delta \blue{d_{\pi}} }_1 \\
&\leq \gamma \norm{ \tilde{G}}_1 \cdot \norm{\Delta \blue{d_{\pi}} }_1 \\
&= \gamma \norm{ (\mathbf{1} + \gamma P_{\tilde{\pi}} + (\gamma P_{\tilde{\pi}})^2 + ...)}_1 \cdot \norm{\Delta \blue{d_{\pi}} }_1 \\
&\leq \gamma (1 + \gamma \norm{P_{\tilde{\pi}}}_1 + \gamma^2 \norm{P_{\tilde{\pi}}^2}_1 + ...) \cdot \norm{\Delta \blue{d_{\pi}} }_1 \\
&= \frac{\gamma}{1-\gamma} \norm{\Delta \blue{d_{\pi}} }_1
\end{align*}
Meanwhile, 
\begin{align*}
\norm{ \Delta \blue{d_{\pi}} }_1 = & \sum_{s^\prime} \left| \sum_{s} \blue{d_{\pi}(s)} \big( p_{\tilde{\pi}}(s^\prime|s) - p_{\pi}(s^\prime|s)\big) \right|  \\
= & \sum_{s^\prime} \left| \sum_{s} \blue{d_{\pi}(s)} \sum_{a} P(s^\prime|s, a) \big( \tilde{\pi}(a|s) - \pi(a|s) \big) \right| \\
\leq & \sum_{s^\prime} \sum_{s} \blue{d_{\pi}(s)} \sum_{a} P(s^\prime|s, a) \left| \tilde{\pi}(a|s) - \pi(a|s) \right| \\
= & \sum_{s} \blue{d_{\pi}(s)} \sum_{a} \left| \tilde{\pi}(a|s) - \pi(a|s) \right| \\
= & \sum_{s,a} \blue{d_{\pi}(s,a)} \left| \frac{\tilde{\pi}(a|s)}{\pi(a|s)}  - 1 \right|
\end{align*}
Thus, $\norm{\red{d_{\tilde{\pi}}} - \blue{d_{\pi}}}_1 \leq \frac{\gamma}{1-\gamma} \mathbb{E}_{(s, a)\sim d_{\pi}}\left| \frac{\tilde{\pi}(a|s)}{\pi(a|s)}  - 1 \right|$. 

Second, we bound $\norm{A_{\pi}^{\tilde{\pi}}}_\infty$. 
\begin{equation*}
\norm{A_{\pi}^{\tilde{\pi}}}_\infty = \max_{s} \left| \mathbb{E}_{a\sim\tilde{\pi}}\big[A_{\pi}(s, a) \big] \right| = \max_{s} \left| \sum_{a}\big[\tilde{\pi}(a|s) A_{\pi}(s, a) \big] \right| \leq \max_{s, a}\left| A_{\pi}(s, a) \right|, 
\end{equation*}
The last inequality is a result of $\tilde{\pi}$ being a probability distribution. 
The resulting term is just $\xi$, as defined in Theorem~\ref{theo:ratio-reg-guarantee}. 

Combined, 
\begin{align*}
    \norm{\red{d_{\tilde{\pi}}} - \blue{d_{\pi}}}_1 \norm{A_{\pi}^{\tilde{\pi}}}_\infty \leq \frac{\xi\gamma}{1-\gamma} \mathbb{E}_{(s, a)\sim d_{\pi}}\left| \frac{\tilde{\pi}(a|s)}{\pi(a|s)}  - 1 \right|.
\end{align*}

Therefore, 
\begin{equation*}
J(\tilde{\pi}) - J(\pi) \geq \frac{1}{1-\gamma} \Bigg[\mathbb{E}_{(s, a)\sim d_\pi}\Big[ \frac{\tilde{\pi}(a|s)}{\pi(a|s)} A_{\pi}(s, a) \Big]
- \frac{\xi \gamma}{1-\gamma}\mathbb{E}_{(s, a)\sim d_{\pi} }\left|\frac{\tilde{\pi}(a|s)}{\pi(a|s)} - 1 \right| \Bigg]. 
\end{equation*}

\end{proof}

\subsection{Ratio-Regularized Policy Optimization (R2PO)}~\label{app:r2po}
According to Theorem~\ref{theo:ratio-reg-guarantee}, one can also formulate the ratio-regularized policy optimization. 
Namely, one can consider the following optimization problem
\begin{equation*}
\max_{\tilde{\pi}} \quad \mathbb{E}_{(s, a)\sim d_\pi}\Big[ \frac{\tilde{\pi}(a|s)}{\pi(a|s)} A_\pi(s, a)\Big]
- C \cdot \mathbb{E}_{(s, a)\sim d_{\pi} }\left|\frac{\tilde{\pi}(a|s)}{\pi(a|s)} - 1 \right|, 
\end{equation*}
where $C$ is the regularization coefficient. 
Optimizing the objective in this optimization amounts to 
altering ratios $\frac{\tilde{\pi}(a|s)}{\pi(a|s)}$ according to $A_{\pi}(s, a)$:
the ratios for $(s, a)$ should be increased when $A_{\pi}(s, a)>0$, and decreased when $A_{\pi}(s, a)<0$. 
At the same time, the ratios should be bounded such that the average deviation from $1.0$ is not too large. 
This idea matches exactly the formulation of ratio clipping PPO in Equation~\ref{equ:ppo-objective}, 
where the ratio clipping takes place within the expectation. 
We call the resulting algorithm Ratio-Regularized Policy Optimization (R2PO).


\subsection{Distributed Early Stopping Policy Optimization}~\label{app:distributed-espo}
The distributed version of ESPO is detailed in Algorithm~\ref{algo:despo}. 
We have experimented with other ways of reducing ratio deviations across all workers 
and have found that synchronously reducing to the minimum leads to better results in practice. 

\begin{algorithm}
    \caption{Distributed ESPO (for each worker)}
    \label{algo:despo}
    \begin{algorithmic}
        \FOR{iterations $i=1, 2,...$}
            \FOR{actor $=1, 2,...,N$}
                \STATE Run policy $\pi$ in environment
                \STATE Computer advantage estimates $\hat{A}_{\pi}$
            \ENDFOR
            \FOR{epoch $=1, 2,...,K$}
                \STATE Sample $M$ samples $\{(s, a)\}$ from previous rollouts.
                \STATE Compute the gradient
                \begin{equation*}
                \nabla_{\theta}\mathcal{L}(\theta) = - \nabla_{\theta} \frac{1}{M}\sum_{s, a}\frac{\tilde{\pi}_{\theta}(a|s)}{\pi(a|s)} \hat{A}_{\pi}(s, a).
                \end{equation*}
                \STATE Send the gradient to all other workers. 
                \STATE Receive the average gradient and update parameters.
                \STATE Estimate deviations $\Delta = \frac{1}{M}\sum_{s, a}\left|\frac{\tilde{\pi}_{\theta}(a|s)}{\pi(a|s)} - 1 \right|$. 
                \STATE Reduce to the minimum of $\Delta$ across all workers. 
                \IF{$\Delta > \delta $}
                    \STATE Break and early stop from optimization epochs. 
                \ENDIF
            \ENDFOR
            \STATE $\pi \leftarrow \tilde{\pi}_{\theta}$. 
        \ENDFOR
    \end{algorithmic}
\end{algorithm}

\subsection{Implementation details}\label{app:implementations}
PPO adds several implementation augmentation to the core algorithm, 
which include the normalization of inputs and rewards, 
the normalization of advantage at each iteration
and generalized advantage estimation~\citep{schulman2015high}.
We adopt the similar augmentations for ESPO and distributed ESPO. 

Specifically, following~\citep{schulman2017proximal,heess2017emergence}, 
we perform the following normalization steps:
\begin{itemize}
    \item We normalize the observations by subtracting the mean 
    and dividing by the standard deviation using statistics aggregated over the entire training process. 

    \item We scale the reward by a running estimate of its standard deviation, 
    which is also aggregated over the entire training process. 

    \item We use per-batch normalization for the advantages. 
\end{itemize}

Networks use \emph{tahn} as the activation function 
and parametrize the mean and standard deviation of a conditional Gaussian distribution over actions.
Network sizes are as follows: 64, 64, action\_dimension. 
Policy network and value network share the same network body. 
Other hyperparameters are given in Table~\ref{tab:hyperparameters}.

\begin{table}
    \centering
    \begin{tabular}{cc}
        \toprule
        hyperparameter & Value \\
        \midrule
        Sampling batch size & 2048 \\
        Epoch batch size & 32 \\
        $\lambda$ & 0.95 \\
        $\gamma$ & 0.99 \\
        Maximum number of optimization epochs & 20 \\
        Entroy coefficient & 0.0 \\
        Initial learning rate & 0.0003 \\
        Learning rate decay & linear \\
        \bottomrule
    \end{tabular}
    \caption{Hyperparameters for all experiments.}
    \label{tab:hyperparameters}
\end{table}

\begin{table}
    \centering
    \begin{tabular}{rcc}
    \toprule
    \multicolumn{1}{r}{\multirow{2}{*}{\textbf{Env.}}}
    & \multicolumn{1}{c}{\multirow{2}{*}{$\mathcal{S}\times\mathcal{A}$}} \\ 
    \multicolumn{1}{c}{}& \multicolumn{1}{c}{}& \\
    \midrule
    Hopper  & $\mathbb{R}^{11}\times\mathbb{R}^{3}$ \\ 
    HalfCheetah  & $\mathbb{R}^{17}\times\mathbb{R}^{6}$ \\ 
    Walker2d & $\mathbb{R}^{17}\times\mathbb{R}^{6}$ \\
    Ant & $\mathbb{R}^{111}\times\mathbb{R}^{8}$ \\
    Humanoid & $\mathbb{R}^{376}\times\mathbb{R}^{17}$ \\
    HumanoidStandup & $\mathbb{R}^{376}\times\mathbb{R}^{17}$ \\\midrule
    dm\_humanoid.stand & $\mathbb{R}^{67}\times\mathbb{R}^{21}$ \\
    dm\_humanoid.walk & $\mathbb{R}^{67}\times\mathbb{R}^{21}$ \\
    dm\_humanoid.run & $\mathbb{R}^{67}\times\mathbb{R}^{21}$ \\
    \bottomrule
    \end{tabular}
    \caption{Control complexity of different domains.}
    \label{tab:control-complexity}
\end{table}

\subsection{Additional experiment results}\label{app:more-results}
\paragraph{PPO ablations.}
Without learning rate annealing, 
ratio clipping PPO may not be able to train policies in general.
Figure~\ref{fig:ppo-no-lr-decay} presents the PPO training curves in Mujoco benchmarks. 
Even with different clipping ranges, 
ratio clipping PPO quickly plateaus in more complicated control tasks, e.g., Ant and Humanoid.
Even with proper learning rate annealing and one manages to improve the performance, 
as presented in Figure~\ref{fig:ppo-lr-decay}, 
it is still non-trivial to find the good values for control with high dimensions, e.g., in Humanoid. 

\begin{figure*}
    \centering
    \includegraphics[width=\linewidth]{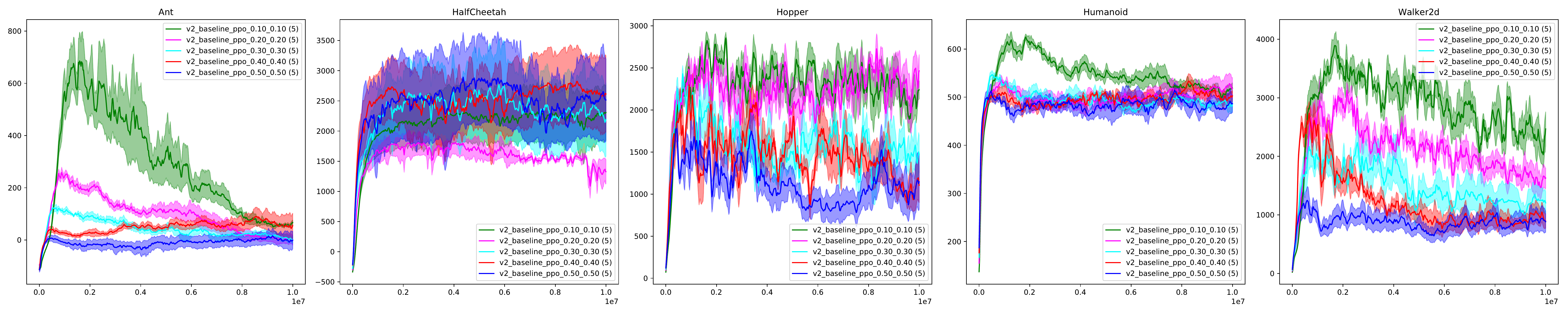}
    \caption{Ratio clipping PPO with fixed learning rate.}
    \label{fig:ppo-no-lr-decay}
\end{figure*}

\begin{figure*}
    \centering
    \includegraphics[width=\linewidth]{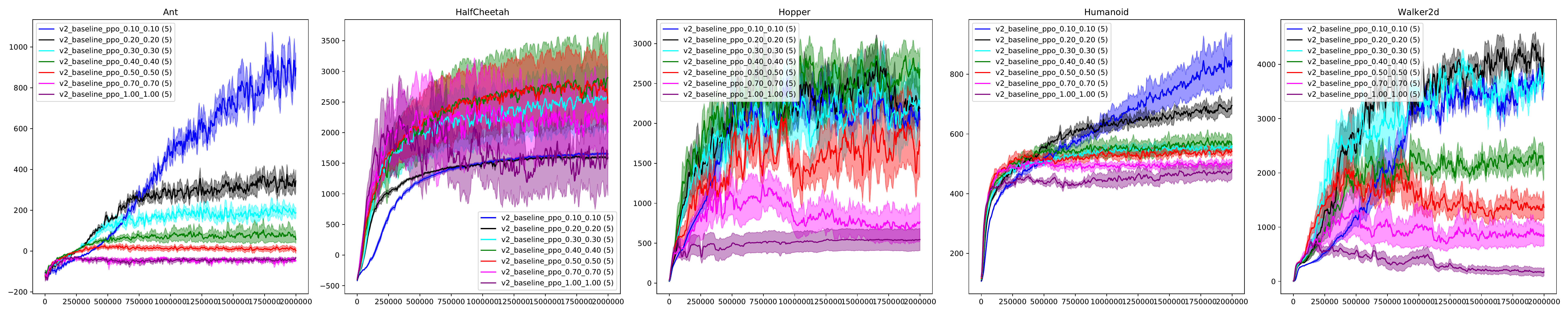}
    \caption{Ratio clipping PPO with proper learning rate annealing.}
    \label{fig:ppo-lr-decay}
\end{figure*}

\paragraph{Ratio-Regularized Policy Optimization (R2PO) ablations}
We also report the R2PO performance and compare it againt PPO and ESPO. 
Specifically, we ablate on the regularization coefficient to investigate how sensitive R2PO is to this hyper-parameter. 
The ablation results are presented in Figure~\ref{fig:ablation-ratio-reg}. 
Figure~\ref{fig:ppo-vs-r2po} presents a comparison between ratio clipping PPO and R2PO. 
When the control becomes very high, e.g., Humanoid, 
R2PO significantly outperforms PPO.

\begin{figure*}
    \centering
    \includegraphics[width=\linewidth]{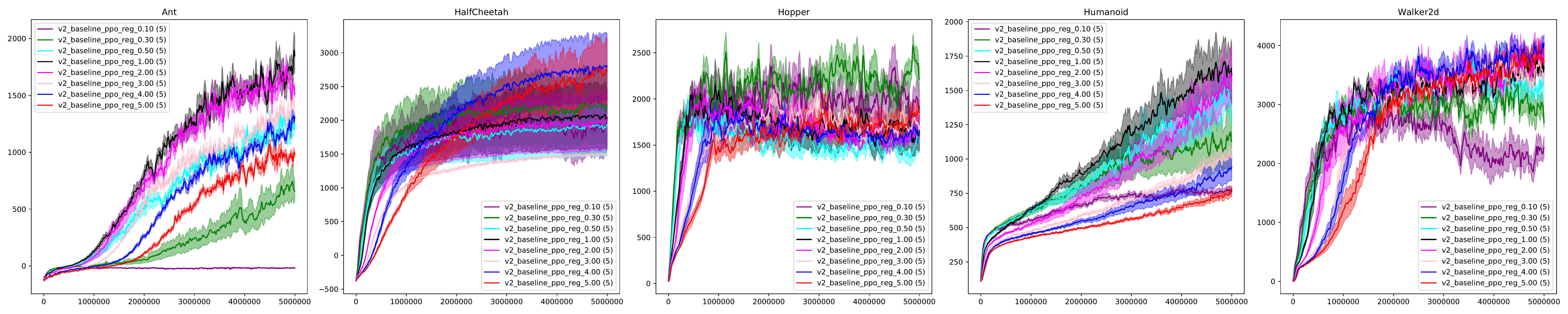}
    \caption{Episodic returns for ablating the regularization coefficient in R2PO.}
    \label{fig:ablation-ratio-reg}
\end{figure*}

\begin{figure*}
    \centering
    \includegraphics[width=\linewidth]{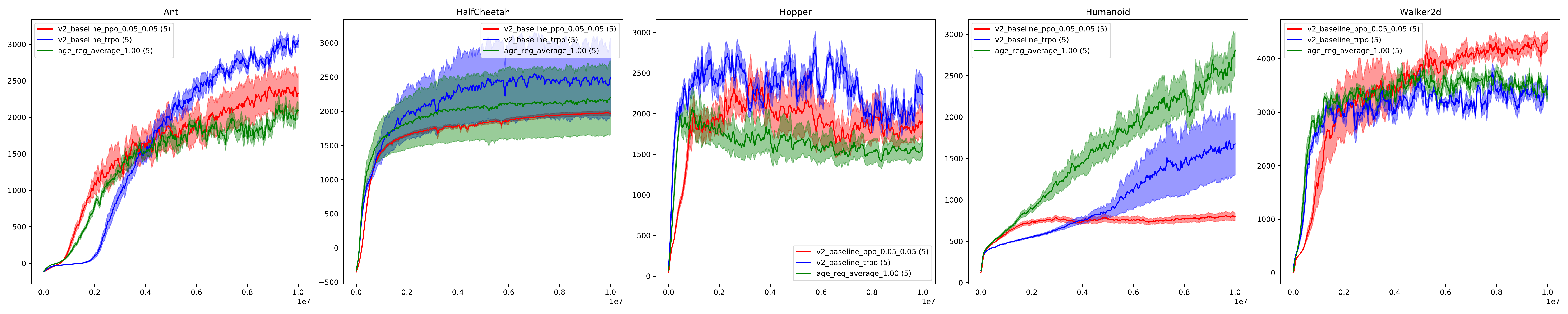}
    \caption{Ratio clipping PPO vs ratio regularized policy optimization.}
    \label{fig:ppo-vs-r2po}
\end{figure*}

\begin{figure}
    \centering
    \includegraphics[width=\linewidth]{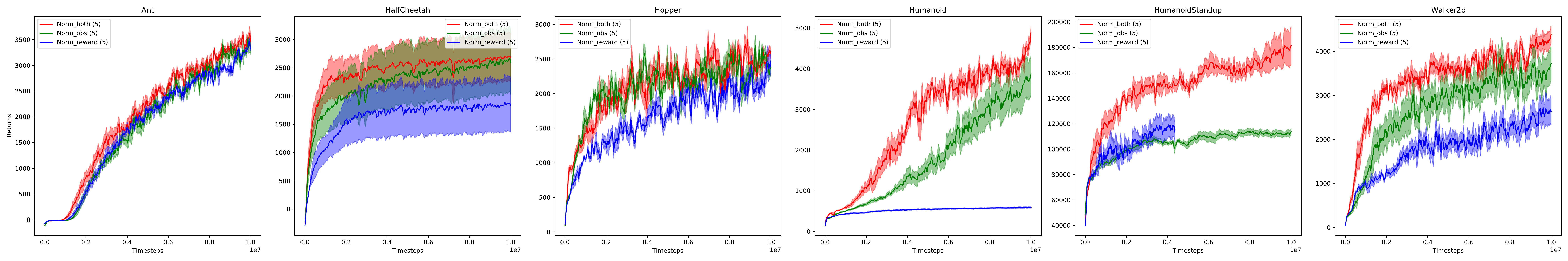}
    \caption{Ablation on different types of normalization on full Mujoco domains, i.e., normalizing observation or reward, or normalizing both.}
    \label{fig:ablate-norm-full}
\end{figure}

\begin{figure*}
    \centering
    \includegraphics[width=0.6\linewidth]{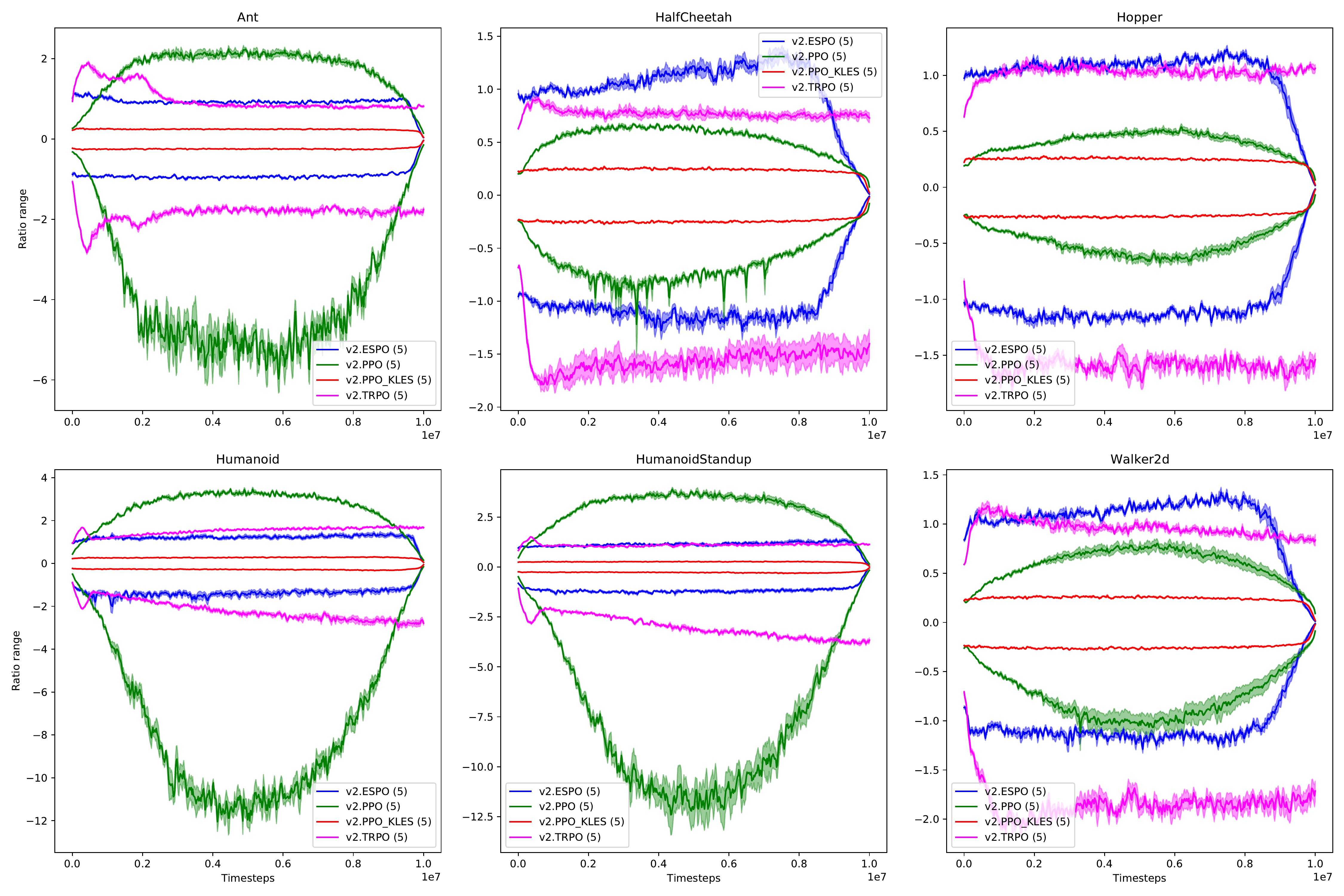}
    \caption{Empirical estimate of the ratio ranges in Mujoco tasks.}
    \label{fig:mujoco-ratios-full}
\end{figure*}

\begin{figure}
    \centering
    \includegraphics[width=0.4\linewidth]{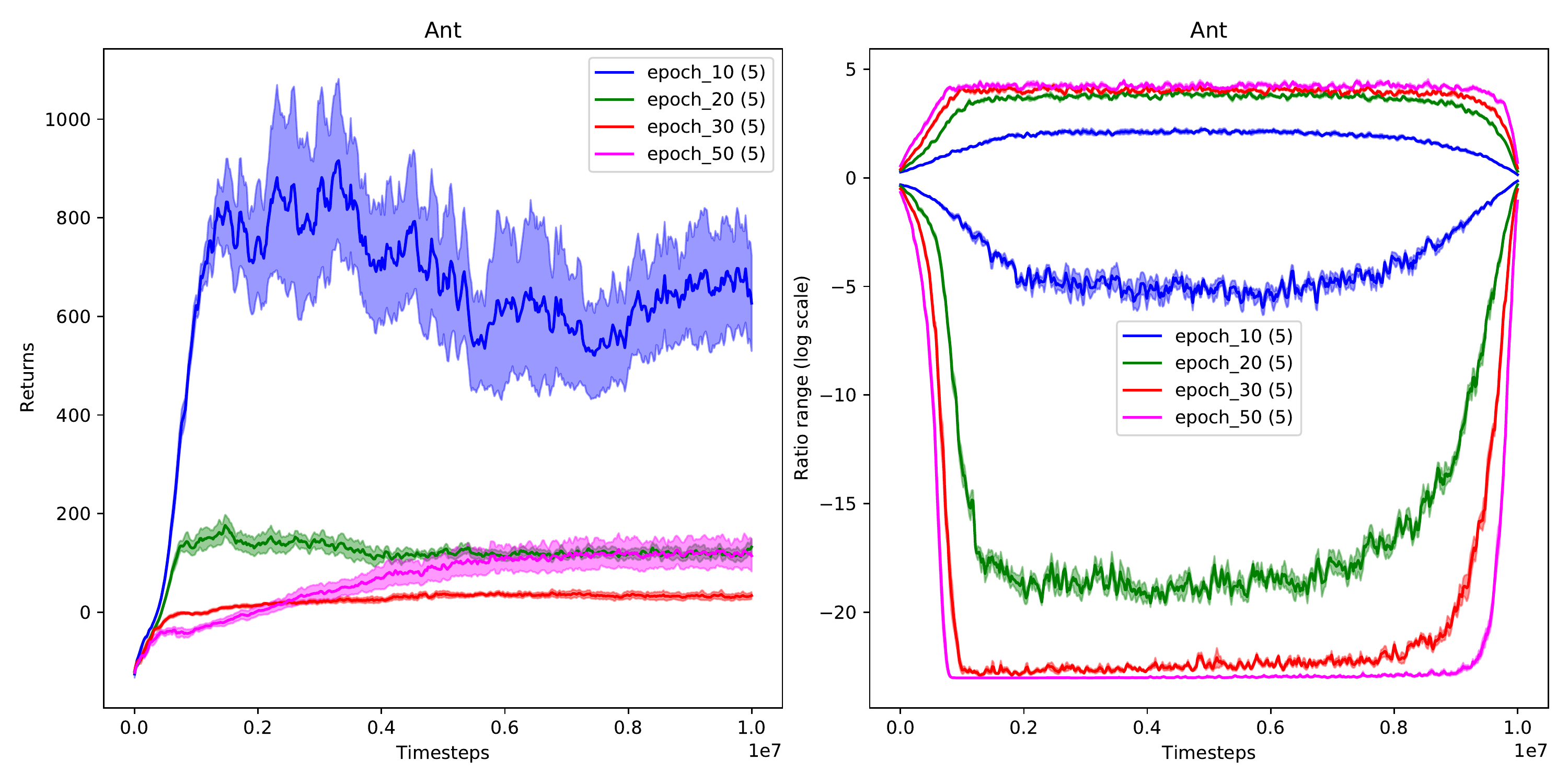}
    \caption{Empirical returns and ratio ranges (log scale) of RC-PPO trained on Ant-v2 with different numbers of optimization epochs; the learning rate decays linearly from $0.0003$ to $0$; 
    the clipping range is set to $0.1$.}
    \label{fig:ant-ratio-range-epoch}
\end{figure}

\begin{figure}
    \centering
    \includegraphics[width=0.4\linewidth]{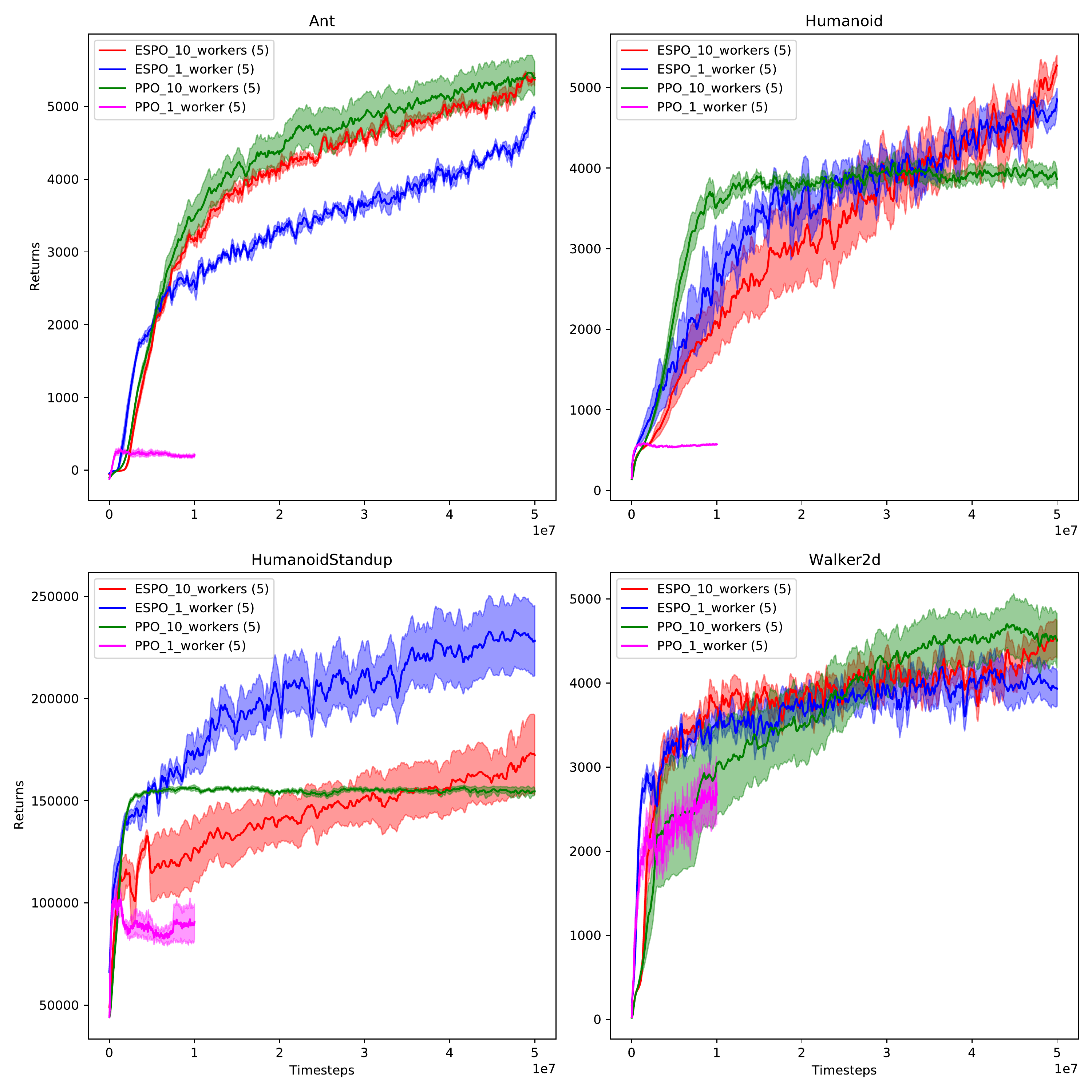}
    \caption{Comparing ESPO with the distributed ratio clipping PPO (DPPO) trained with 10 workers, i.e., PPO-10-workers. 
    The performance of PPO trained in a single process is denoted as PPO-1-worker. 
    ESPO is also scaled up to 10 workers, denoted as ESPO-10-workers.}
    \label{fig:espo-vs-dppo}
\end{figure}

\begin{figure*}
    \centering
    \includegraphics[width=0.8\linewidth]{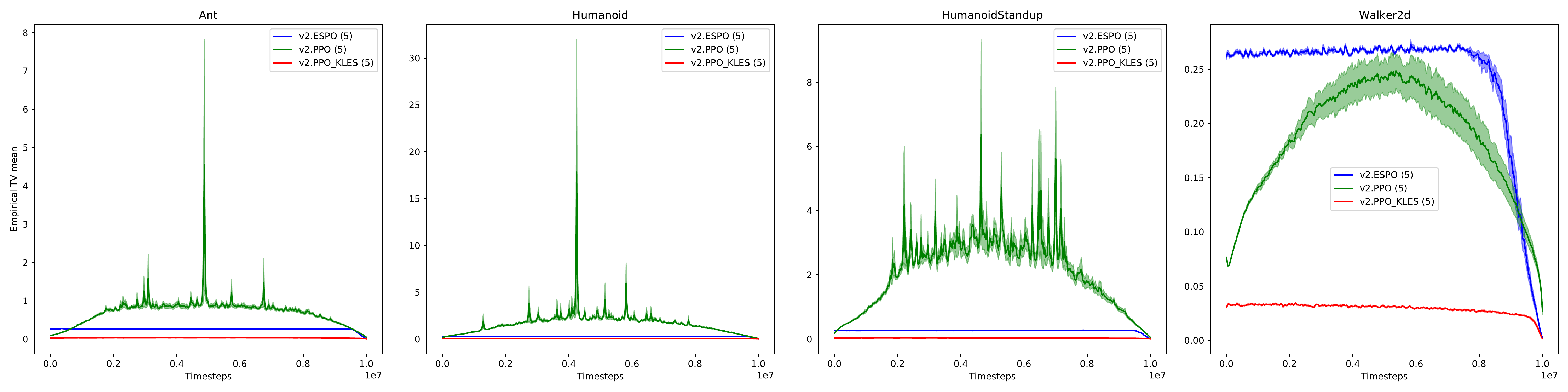}
    \caption{Empirical estimate of the expected ratio deviations in Mujoco tasks.}
    \label{fig:mujoco-empirical-tv}
\end{figure*}

\begin{figure*}
    \centering
    \includegraphics[width=0.6\linewidth]{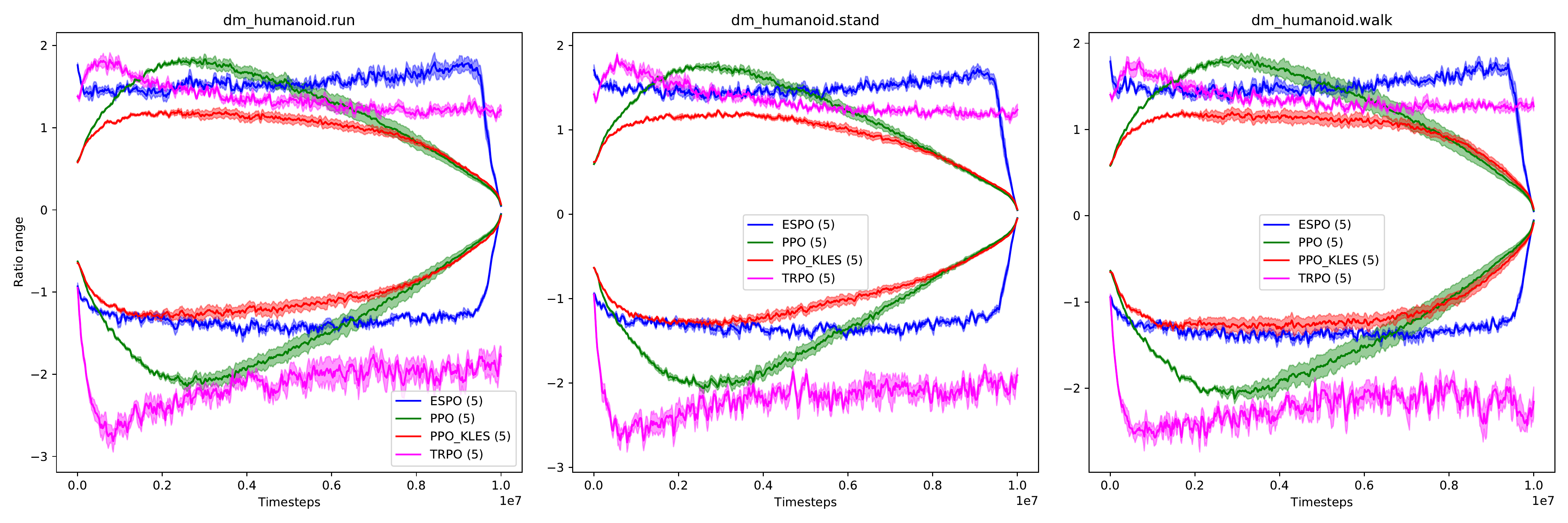}
    \caption{Empirical estimate of the ratio ranges in DeepMind control suite tasks.}
    \label{fig:dm-ratios-full}
\end{figure*}

\begin{figure*}
    \centering
    \includegraphics[width=0.6\linewidth]{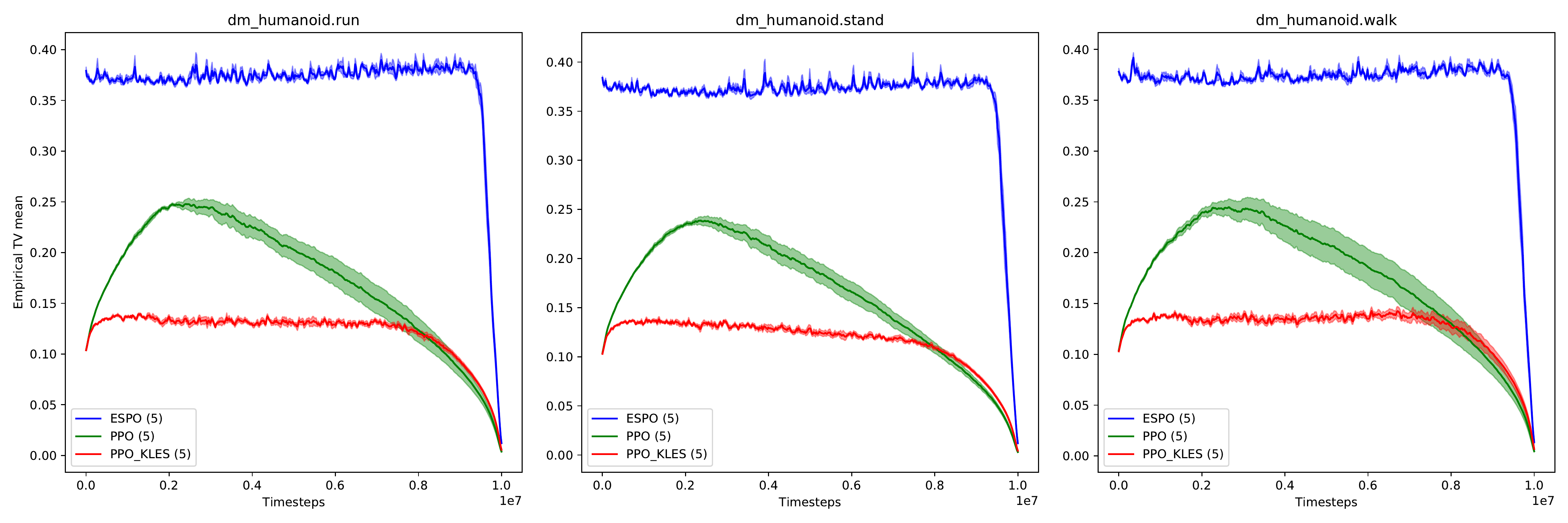}
    \caption{Empirical estimate of the expected ratio deviations in DeepMind control suite tasks.}
    \label{fig:dm-tv-full}
\end{figure*}

\begin{figure*}
    \centering
    \includegraphics[width=0.6\linewidth]{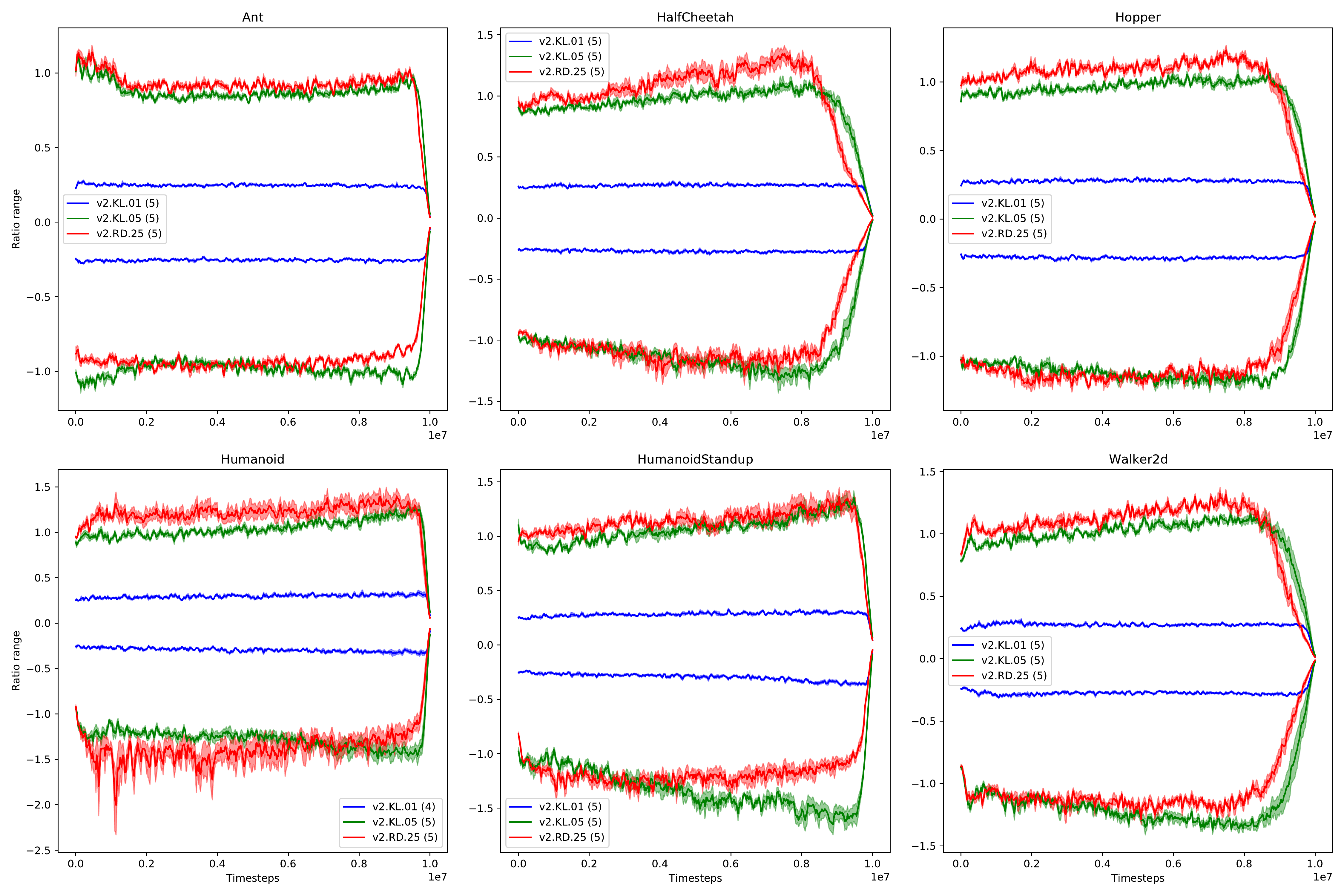}
    \caption{Empirical estimate of the ratio ranges for different ways of early stopping.}
    \label{fig:rd-vs-kl-ratios}
\end{figure*}

\begin{figure*}
    \centering
    \includegraphics[width=0.6\linewidth]{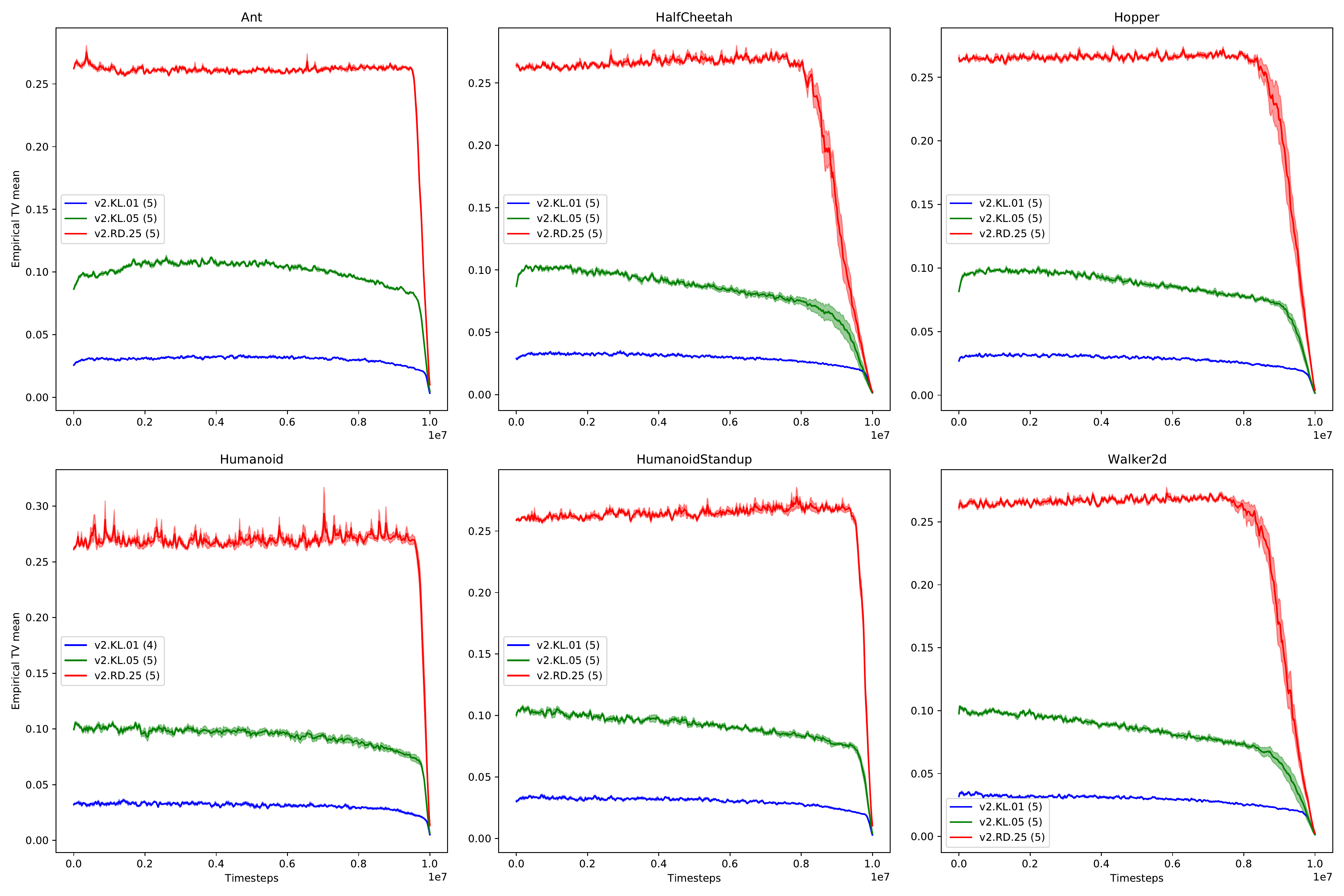}
    \caption{Empirical estimate of the expected ratio deviations for different ways of early stopping.}
    \label{fig:rd-vs-kl-tv}
\end{figure*}

\begin{figure*}
    \centering
    \includegraphics[width=\linewidth]{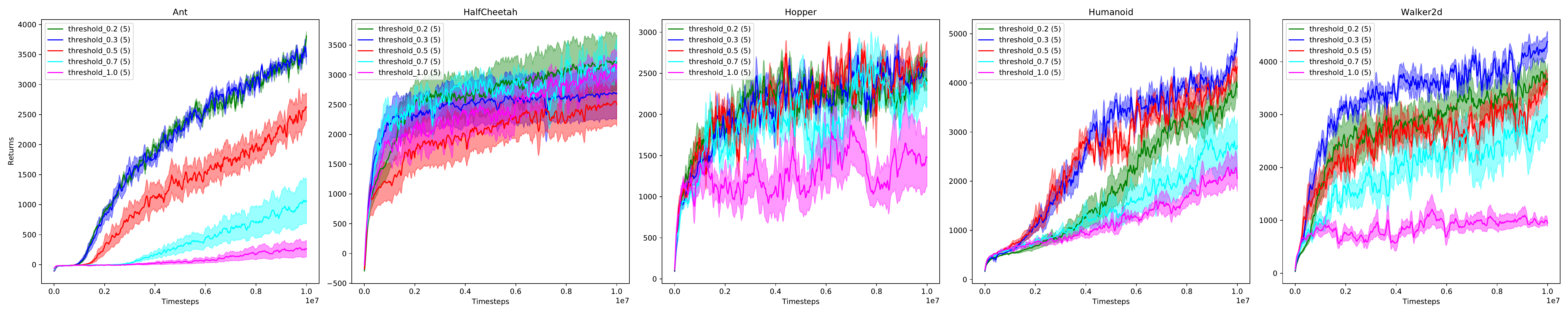}
    \caption{Ablation on the threshold values on full Mujoco domains, i.e., when to drop out from the multiple optimization epochs.}
    \label{fig:ablate-threshold-full}
\end{figure*}

\subsection{Ablating on the normalization}

We also find that normalization is crucial to the strong performance of ESPO (also to PPO and TRPO) in some domains. 
We present an ablation study on two types of normalization in Figure~\ref{fig:ablate-norm}:
observation normalization and reward normalization. 
Specifically, normalization has little impact on  performance in Ant but dramatically affects performance in Humanoid. 
Both domains have roughly the same number of dimensions in the observation space: 
$29$ for Ant and $33$ for Humanoid (see Table~\ref{tab:hyperparameters} for more details).
Nevertheless, neither observation nor reward normalization worsens performance. 
Thus, it may be better to normalize both in practice. 
Also see Figure~\ref{fig:ablate-norm-full} 
for ablation results on the full set of Mujoco domains. 

\end{document}